\documentclass[11pt]{article}

\usepackage[
  shownumpages,
  bgcolor={250,248,241},
  braincolor={133,76,101},
  linkcolor={133,76,101},
  citecolor={133,76,101},
  urlcolor={133,76,101},
  citingstyle=authoryear,
  bibliostyle=unsrtnat,
  bibfile=refs
]{Styles/brainlab}

\usepackage{microtype}
\usepackage{mathtools}
\usepackage{bm}
\usepackage{multirow}
\usepackage{tabularx}
\usepackage{longtable}
\usepackage{threeparttable}
\usepackage{subcaption}
\usepackage{enumitem}
\usepackage{wrapfig}

\definecolor{HWPrimary}{HTML}{854C65}
\definecolor{HWSecondary}{HTML}{F0C987}
\definecolor{HWSoft}{HTML}{FBEDD6}
\definecolor{HWBg}{HTML}{FAF8F1}
\definecolor{HWAccent}{HTML}{854C65}

\usepackage{tikz}
\usetikzlibrary{calc,positioning,arrows.meta,shapes,fit,decorations.pathreplacing}
\usepackage{pgfplots}
\pgfplotsset{compat=1.18}
\usepgfplotslibrary{fillbetween}

\graphicspath{{plots/}{./plots/}}

\setlist[itemize]{leftmargin=1.2em,itemsep=0.35em,topsep=0.35em}
\setlist[enumerate]{leftmargin=1.4em,itemsep=0.35em,topsep=0.35em}

\DeclareBrainTcbTheorem{lemma}{Lemma}
\DeclareBrainTcbTheorem{proposition}{Proposition}
\DeclareBrainTcbTheorem{remark}{Remark}
\DeclareBrainTcbTheorem{example}{Example}
\DeclareBrainTcbTheorem{note}{Note}
\newtheorem*{theorem*}{Theorem}

\DeclareMathOperator*{\argmin}{arg\,min}

\setbrainmeta{
  title={LoRA-TSD: Tangent-Space Spectral Descent for LoRA via Muon-Style Updates},
  authors={
    Dmitrii Andriianov\textsuperscript{1},
    Andrey Veprikov\textsuperscript{1,2},
    Aleksandr Beznosikov\textsuperscript{1,3}
  },
  affiliations={
    \textsuperscript{1}Basic Research of Artificial Intelligence Laboratory (BRAIn Lab)\\
    \textsuperscript{2}SB AI Lab\\
    \textsuperscript{3}Innopolis University
  },
  abstract={
  Low-rank adaptation (LoRA) is the standard way to fine-tune large models, yet when
  its two factors are trained independently, the update ignores the geometry
  of the low-rank weight change it induces. We introduce LoRA-TSD, an optimizer
  that treats every LoRA step as a tangent vector of the fixed-rank matrix
  manifold and takes the spectral-norm steepest-descent step of Muon inside that
  tangent space, mapping the result back to the factors through a retraction
  native to the LoRA parametrization. The step avoids expensive operations on full
  weight matrices, and its retraction is up to $2.8\times$ cheaper than the
  truncated-SVD retraction used by prior manifold methods. We prove that the Frobenius-norm
  version of our surrogate recovers LoRA-Pro, and we identify the
  tangent-projected gradient, the Riemannian gradient of the manifold, as the
  stationarity measure natural to LoRA training and computable from the factor
  gradients alone. Under this measure we give the first global convergence
  guarantees for both LoRA-Pro and LoRA-TSD, with rates that drive the
  factor-gradient norms to zero. Across six commonsense and
  natural-language-inference benchmarks with Llama-3.2-1B, Llama-3.1-8B and
  Qwen3-32B, LoRA-TSD outperforms every competing LoRA optimizer and stays
  robust to the adapter rank. Code is available at
  \url{https://github.com/brain-lab-research/LoRA-TSD}.
  },
}

\begin{document}
\begin{mainpart}

\section{Introduction}
\label{sec:intro}
Full fine-tuning adapts all parameters of a pretrained language model, but its
memory and storage costs become prohibitive at scale. Training requires
gradients and optimizer states for every parameter, and each downstream task
produces a separate model copy. Parameter-efficient fine-tuning (PEFT) reduces
these costs by
updating a small set of parameters while keeping the pretrained weights
frozen~\citep{han2024peft}. Among PEFT methods, low-rank adaptation
(LoRA) has become a common approach to adapting foundation models. Its compact
task-specific updates can be merged into the backbone without additional
inference latency, while multiple adapters can share the same frozen
model~\citep{hu2022lora}. This structure is also used in deployed systems,
including multi-adapter serving and domain-specific language models for
customer-facing applications~\citep{nguyen2024multitenant,scozzafava2025impel}.

For a frozen weight $W_0\in\mathbb{R}^{m\times n}$, LoRA represents the update
as $\Delta W=BA$, where
$B\in\mathbb{R}^{m\times r}$, $A\in\mathbb{R}^{r\times n}$, and
$r\ll\min(m,n)$. This parameterization has motivated a broad line of subsequent
work. Existing methods reduce training memory, allocate the rank budget
adaptively, improve initialization, modify factor scaling or weight
parameterization, and select the most useful
adapters~\citep{dettmers2023qlora,zhang2023adalora,meng2024pissa,hayou2024loraplus,liu2024dora,veprikov2026weightlora}.

Despite this progress, the optimization of the LoRA factors remains less
settled. Standard optimizers such as Adam~\citep{kingma2015adam} update $A$ and
$B$ as independent parameter blocks, even though the model depends on their
product $BA$. Consequently, a well-scaled step in factor space can produce a
poorly scaled or misaligned update in weight space. LoRA-specific optimizers
address this mismatch by correcting the induced weight update, removing its
dependence on the chosen factorization, or optimizing the low-rank weight
directly through fixed-rank manifold
geometry~\citep{lora-pro,tastan2025loft,yen2024lorarite,bonnabel2013stochastic,zhang2024riemannian,riemannion,cesista2026loramuon}.
Section~\ref{sec:rel_work} reviews these approaches and explains why applying a
full-matrix optimizer such as Muon within the LoRA parameterization remains
nontrivial.

We propose LoRA-TSD (Tangent-Space Spectral Descent), which brings the
spectral update of Muon~\citep{muon} to LoRA. Muon balances the singular
directions of a matrix update instead of allowing a few directions to dominate.
LoRA-TSD computes the corresponding update for the low-rank weight and maps it
back to $A$ and $B$. This mapping gives a factor-induced retraction that follows
directly from the LoRA parametrization. Both our retraction and the
truncated-SVD retraction used by prior manifold methods admit low-rank
implementations. Ours avoids the truncated SVD and is faster in practice.

\paragraph{Contributions.}
\begin{itemize}
  \item We develop LoRA-TSD, a LoRA-native extension of the tangent spectral
    iteration proposed by Riemannion~\citep{riemannion}. It operates through standard LoRA factors
    and uses a factor-induced retraction without materializing $m\times n$
    matrices.
  \item In the $L$-smooth setting, we prove convergence to stationarity for
    momentum-free LoRA-TSD.
  \item We identify LoRA-Pro~\citep{lora-pro} as the Frobenius counterpart of
    LoRA-TSD and prove its convergence to stationarity.
  \item Across the evaluated benchmarks, LoRA-TSD outperforms generic and
    LoRA-specific baselines and is the most stable method across adapter ranks.
\end{itemize}


\section{Related Work}
\label{sec:rel_work}
\paragraph{LoRA-specific optimizers.}
The induced weight update $\Delta(BA) = \Delta B\,A + B\,\Delta A$ is bilinear in
the factors, so a step that is well-scaled for $A$ and $B$ can be misaligned in
weight space: it depends on the factor norms and conditioning and generally
differs from direct descent on $BA$.
LoRA-specific optimizers address this mismatch. One line of work improves
the parametrization or its
initialization. LoRA+~\citep{hayou2024loraplus} assigns separate learning rates
to the two factors, removing a width-dependent
inefficiency of naive LoRA. DoRA~\citep{liu2024dora} decomposes each pretrained
weight into separately trainable magnitude and directional components, using
LoRA to parameterize the directional update. PiSSA~\citep{meng2024pissa} and
LoRA-GA~\citep{wang2024loraga} instead pick the initial factors from the leading
singular directions of the weight or of its gradient, aligning the first LoRA
step with full fine-tuning. A second line keeps the parametrization fixed and
corrects the update itself. LoRA-Pro~\citep{lora-pro} is the central example.
Let $\mathcal{L}$ denote the training loss. LoRA-Pro chooses the factor updates
\begin{equation}
  (\Delta A^{\star}, \Delta B^{\star})
  = \argmin_{\Delta A, \Delta B}
  \bigl\|\nabla_W\mathcal{L} - \bigl(B\Delta A + \Delta B A\bigr)\bigr\|_F^2 .
  \label{eq:lora-pro-projection}
\end{equation}
LoFT~\citep{tastan2025loft} pursues a related
correction at the optimizer level by projecting the Adam moments into the
low-rank subspace. The same factor ambiguity also creates an optimization
problem: if two factorizations of the same adapter lead to different
weight-space steps, training can become sensitive to the arbitrary scale and
rotation of the LoRA factors, with one factor moving while the other stays nearly
fixed~\citep{yen2024lorarite}. LoRA-Rite formalizes the desired consistency as
transformation invariance, meaning that replacing $(B,A)$ by $(BR^{-1},RA)$ for
any invertible $R \in \mathbb{R}^{r \times r}$ does not change the induced
update to $BA$. LoRA-TSD satisfies the same invariance, but chooses the target
direction through spectral-norm tangent geometry. In the Frobenius-norm case,
our formulation recovers the LoRA-Pro direction, with the step scale set by the
optimizer learning rate.

\paragraph{Orthogonalized and spectral-norm optimizers.}
Muon~\citep{muon} orthogonalizes the momentum update before the step, which
equalizes its singular values and prevents it from collapsing onto a few dominant
directions. \citet{bernstein2024old} identify this update as steepest
descent under the spectral norm, placing Muon within a broader family of
norm-constrained optimizers. Muon assumes access to the full weight matrix, so
transferring its geometry to the low-rank setting is not immediate. LoRA-TSD does
so by solving the spectral-norm oracle inside the tangent space of the fixed-rank
manifold.

\paragraph{Riemannian optimization for LoRA.}
Riemannian SGD projects the Euclidean gradient onto the tangent space and
retracts the resulting step back to the
manifold~\citep{bonnabel2013stochastic,absil2008optimization,vandereycken2013lowrank}.
\citet{zhang2024riemannian} keep this scheme but replace the Euclidean metric
with an $r \times r$ preconditioner that rescales each factor update by the Gram
matrix of the other factor. Both descend along the ordinary tangent gradient,
whereas Riemannion and LoRA-Muon, like our method, orthogonalize this tangent
step to obtain a spectral-norm update.

The closest work to ours is Riemannion~\citep{riemannion}, which applies Muon-like orthogonalized updates
to low-rank adapters using the Riemannian geometry of fixed-rank matrices. We share the
tangent-space motivation but differ in how the spectral-norm tangent
update is approximated and how the tangent direction is mapped back
to the factors. Section~\ref{sec:method} makes this comparison precise.

A closely related optimizer, LoRA-Muon~\citep{cesista2026loramuon}, develops the same
spectral steepest descent view of the LoRA tangent space, together with a gauge-invariance
analysis and a matched split weight-decay rule. To make the tangent problem tractable in
closed form, it decouples the two tangent components $\Delta B\,A$ and $B\,\Delta A$,
splitting the trust region evenly between them and orthogonalizing each half with a single
matrix sign step. The decoupling buys a closed form but discards the interaction between the
two components, and the even split approximates the spectral-norm constraint only while the
two stay aligned during training. The resulting step is a single closed-form expression with
no mechanism to refine the direction toward the
constrained spectral solution. LoRA-TSD retains the interaction through the tangent
projector and reaches this solution by the alternating projection of
Eq.~\eqref{eq:lora-alt-proj-iteration}.

Concurrently, PoLoRA~\citep{ghosh2026polora} derives a product-aware LMO with a
curvature-weighted spectral constraint. Like LoRA-Muon, it decouples the
constraint on $B\Delta A+\Delta B A$ into two factorwise constraints, then adds
curvature preconditioning and magnitude control. LoRA-TSD instead retains the
coupled tangent update.


\section{Method}
\label{sec:method}
\subsection{Notation}
\label{subsec:notation}
We derive the update for a single adapted layer. Let
$W \in \mathbb{R}^{m \times n}$ be its weight matrix, and let
$\mathcal{L}(W)$ denote the training loss with all other parameters fixed. The
same update is applied independently to each adapted layer.

\paragraph{LoRA parametrization and gradients.}
LoRA keeps the pretrained matrix $W_0$ fixed and trains a rank-$r$ residual,
\begin{equation}
  W_k = W_0 + B_k A_k,
\end{equation}
where $B_k \in \mathbb{R}^{m \times r}$, $A_k \in \mathbb{R}^{r \times n}$,
and $r \ll \min(m,n)$. At iteration $k$, write
\begin{equation}
  G_W^k = \nabla_W \mathcal{L}(W_k), \qquad
  G_A^k = \nabla_A \mathcal{L}(W_k), \qquad
  G_B^k = \nabla_B \mathcal{L}(W_k).
\end{equation}
The chain rule gives the relations between the full-space gradient and the
observable factor gradients,
\begin{equation}
  G_A^k = B_k^\top G_W^k, \qquad
  G_B^k = G_W^k A_k^\top .
  \label{eq:lora-factor-gradients}
\end{equation}
The matrix $G_W^k$ is the gradient we would obtain under full fine-tuning, but
LoRA training exposes only $G_A^k$ and $G_B^k$.

\paragraph{Spectral smoothness.}
Muon~\citep{muon} is steepest descent in the spectral-norm
geometry~\citep{bernstein2024old}, so we use the same norm for the local model.
Assume that $\nabla_W \mathcal{L}(W)$ is $1/\lambda$-Lipschitz continuous with
respect to the spectral norm. Then, for any perturbation $\Delta W_k$,
\begin{equation}
  \mathcal{L}(W_k + \Delta W_k)
  \le \mathcal{L}(W_k)
  + \langle \nabla \mathcal{L}(W_k), \Delta W_k \rangle
  + \frac{1}{2\lambda}\|\Delta W_k\|_2^2 .
  \label{eq:spectral-smoothness}
\end{equation}

\paragraph{First-order LoRA steps.}
A transition from $(B_k,A_k)$ to
$(B_k+\Delta B_k,A_k+\Delta A_k)$ induces the weight change
\begin{equation}
  \Delta W_k^{\mathrm{LoRA}}
  = \Delta B_k A_k + B_k \Delta A_k + \Delta B_k \Delta A_k .
  \label{eq:lora-induced-step}
\end{equation}
The first two terms are linear in the factor increments. The product
$\Delta B_k \Delta A_k$ is second order. Following standard LoRA optimizer
derivations~\citep{lora-pro,tastan2025loft,riemannion}, we choose the direction
from the first-order transition
$\Delta B_k A_k + B_k \Delta A_k$.

\paragraph{Fixed-rank manifold and tangent space.}
The reachable LoRA residuals $BA$ have rank at most $r$. We work on the
rank-exact stratum because every matrix of rank at most $r$ can be approximated
arbitrarily well by rank-$r$ matrices. The residual then lies on the smooth
fixed-rank manifold
$\mathcal{M}_r \subset \mathbb{R}^{m \times n}$~\citep{absil2008optimization,vandereycken2013lowrank}.
The affine shift by $W_0$ does not change the tangent geometry. The first-order
LoRA updates at $W_k=W_0+B_kA_k$ form the tangent space
\begin{equation}
  \mathcal{T}_{W_k}\mathcal{M}_r
  = \bigl\{\Delta B_k A_k + B_k \Delta A_k
  \mid \Delta B_k \in \mathbb{R}^{m \times r},\
  \Delta A_k \in \mathbb{R}^{r \times n}\bigr\}.
  \label{eq:lora-tangent-space}
\end{equation}
After the second-order term is dropped, this is the exact search space of LoRA.
It depends only on the column space of $B_k$ and the row space of $A_k$, not on a
particular scaling of the factors.
When the iteration index is clear, write
$T=\mathcal{T}_{W_k}\mathcal{M}_r$ and denote by $\mathcal{P}_T$ the orthogonal
projector onto this tangent space.

\paragraph{Linear minimization oracle.}
\citet{bernstein2024old} derive connections between many optimizers and linear
minimization oracles over norm balls. For a norm ball $\mathcal{B}$ and a
gradient $G$, the oracle returns
\[
  \operatorname{LMO}_{\mathcal{B}}(G)
  \in \argmin_{X \in \mathcal{B}} \langle G, X\rangle .
\]
For the spectral-norm ball, the oracle is the negative polar factor of $G$,
which is the matrix-sign direction used by Muon.

\subsection{Deriving the optimization problem}
Apply the spectral smoothness bound in Eq.~\eqref{eq:spectral-smoothness} to the
first-order LoRA transition:
\begin{equation}
\begin{aligned}
  \mathcal{L}\bigl(W_k + \Delta B_k A_k + B_k \Delta A_k\bigr)
  \le \mathcal{L}(W_k)
  + \langle G_W^k, \Delta B_k A_k + B_k \Delta A_k \rangle
  + \frac{1}{2\lambda}
  \|\Delta B_k A_k + B_k \Delta A_k\|_2^2 .
  \label{eq:lora-second-order-full}
\end{aligned}
\end{equation}
This yields the local surrogate
\begin{equation}
  \min_{\Delta A_k,\Delta B_k}
  \left(
  \langle G_W^k, \Delta B_k A_k + B_k \Delta A_k \rangle +
  \frac{1}{2\lambda}\|\Delta B_k A_k + B_k \Delta A_k\|_2^2
  \right).
  \label{eq:lora-surrogate-objective}
\end{equation}
The factor increments are not unique. Different pairs
$(\Delta B_k,\Delta A_k)$ can induce the same weight-space change, and the
network only sees this change. We therefore optimize directly over
$X=\Delta B_k A_k+B_k\Delta A_k$. Using
Eq.~\eqref{eq:lora-tangent-space}, Eq.~\eqref{eq:lora-surrogate-objective}
becomes
\begin{equation}
  \min_{X \in \mathcal{T}_{W_k}\mathcal{M}_r}
  \left(\langle G_W^k, X \rangle + \frac{1}{2\lambda} \|X\|_2^2\right).
  \label{eq:lora-tangent-surrogate}
\end{equation}
The tangent-space surrogate of Eq.~\eqref{eq:lora-tangent-surrogate} is the
common formulation behind several LoRA optimizers. Its Frobenius-norm version
recovers the LoRA-Pro direction, as proved in Section~\ref{sec:theory}.
LoRA-TSD instead retains the spectral norm, like
Riemannion~\citep{riemannion}, and uses the LMO formulation from
Section~\ref{subsec:notation} to obtain a Muon-style tangent update.

\subsection{Tangent-space spectral descent}
\label{subsec:tangent-spectral-descent}
We now solve Eq.~\eqref{eq:lora-tangent-surrogate} in the spectral norm, matching
the geometry of Muon. The quadratic term sets the step scale, while the linear
term sets the descent direction. Since the spectral norm is homogeneous, the
direction of the penalized problem can be obtained from the corresponding
trust-region form: for a fixed radius, minimize the same linear term over the
spectral-norm ball. This direction-scale separation is standard in normalized
deep-learning optimizers. Sign and matrix-sign methods first choose a
norm-steepest direction and then set the step length with a learning rate.
\citet{bernstein2024old} derive this LMO view for a broad class of optimizers,
including the spectral-norm update underlying Muon.

For LoRA, the trust region must also respect the tangent constraint
$X \in \mathcal{T}_{W_k}\mathcal{M}_r$, giving
\begin{equation}
  \min_{\substack{X \in \mathcal{T}_{W_k}\mathcal{M}_r \\ \|X\|_2 \le \rho}}
  \langle G_W^k, X \rangle .
  \label{eq:lora-constrained-tangent}
\end{equation}
This move isolates the descent direction from the step scale. The scale is then
set by the learning rate, while the convergence analysis only requires a
positive descent margin for the direction returned by the constrained oracle.
We approximate the constrained oracle with alternating
projections~\citep{bauschke1996projection}: a tangent projection followed by an
orthogonalization toward the spectral-norm ball,
\begin{equation}
  X_{t+1} = \mathcal{P}_T\bigl(\mathrm{msign}(X_t)\bigr),
  \qquad X_0 = \mathcal{P}_T(G_W^k).
  \label{eq:lora-alt-proj-iteration}
\end{equation}

where $\mathrm{msign}(Z) = UV^\top$ for $Z = U\Sigma V^\top$, the polar factor.
The limit $X_\infty$ is the flattest matrix in $T$ consistent with
$\mathcal{P}_T(G_W^k)$. This procedure is close to the tangent-space
orthogonalization of Riemannion~\citep{riemannion}, but applies the projection
iteratively as a refinement toward Eq.~\eqref{eq:lora-constrained-tangent}.

Appendix~\ref{app:experiments} analyzes the update spectrum and shows that its
first $2r$ singular values are almost equalized. This is the analogue of Muon
for the LoRA tangent space. Theorem~\ref{thm:lora-tsd-convergence} establishes
convergence for this update under the stated assumptions.

\paragraph{Closed form for $\mathcal{P}_T(G_W^k)$.}
We omit the iteration index $k$ for brevity. For $P_B = B(B^\top B)^{-1}B^\top$ 
and $P_A = A^\top(AA^\top)^{-1}A$ the projector onto the tangent space is
\begin{equation}
  \mathcal{P}_T(Z) = P_B Z + Z P_A - P_B Z P_A.
  \label{eq:lora-tangent-projector}
\end{equation}
Substituting $Z = G_W$, the unobserved full-space gradient cancels, and
$\mathcal{P}_T(G_W)$ is expressed purely through the factor gradients $G_A = B^\top G_W$
and $G_B = G_W A^\top$:
\begin{equation}
  \mathcal{P}_T(G_W)
  = B(B^\top B)^{-1} G_A
  + G_B(AA^\top)^{-1} A
  - B(B^\top B)^{-1}(G_A A^\top)(AA^\top)^{-1} A .
  \label{eq:lora-tangent-g-closed-form}
\end{equation}
Thus $\mathcal{P}_T(G_W)$ can be computed without forming $G_W$, which LoRA
does not expose (Section~\ref{subsec:notation}).

\paragraph{Reconstructing the LoRA factor updates.}
The alternating projections in Eq.~\eqref{eq:lora-alt-proj-iteration}
yield a full-space update $\Delta W = \Delta B A + B \Delta A \in T$. To apply
the update in the LoRA factors, we recover $\Delta A$ and $\Delta B$. The split
is not unique, so we use the following canonical choice.
Let $P_A = A^\top(AA^\top)^{-1}A$ be the orthogonal projector onto
$\mathrm{row}(A)$. As $\mathrm{row}(\Delta B A) \subseteq \mathrm{row}(A)$ we set
\begin{equation}
  \Delta B A = \Delta W P_A = \Delta W A^\top (AA^\top)^{-1} A ,
  \label{eq:lora-delta-ba}
\end{equation}
so that
\begin{equation}
  \Delta B = \Delta W A^\top (AA^\top)^{-1} .
  \label{eq:lora-delta-b}
\end{equation}
The residual is assigned to $B \Delta A$,
\begin{equation}
  B \Delta A = \Delta W - \Delta B A ,
  \label{eq:lora-delta-ba-residual}
\end{equation}
and hence
\begin{equation}
  \Delta A = (B^\top B)^{-1} B^\top (\Delta W - \Delta B A) .
  \label{eq:lora-delta-a}
\end{equation}


\begin{wrapfigure}[13]{r}{0.55\textwidth}
  \vspace{-10pt}
  \begin{algorithm}[width=\linewidth]{LoRA-TSD (conceptual form)}
    \label{alg:lora-tsd-simple}
    \begin{algorithmic}[1]
      \For{$k = 0, 1, \dots$}
        \State $G_A^k \gets \nabla_A \mathcal{L}(W_k)$, \;
          $G_B^k \gets \nabla_B \mathcal{L}(W_k)$
        \State $X_0 \gets \mathcal{P}_T(G_W^k)$
          \Comment{Eq.~\eqref{eq:lora-tangent-g-closed-form}}
        \For{$t = 0, \dots, \tau - 1$}
          \State $X_{t+1} \gets \mathcal{P}_T\bigl(\mathrm{msign}(X_t)\bigr)$
            \Comment{Eq.~\eqref{eq:lora-alt-proj-iteration}}
        \EndFor
        \State $\Delta W_k \gets -\eta X_\tau$
        \State $\Delta B_k \gets \Delta W_k A_k^\top (A_k A_k^\top)^{-1}$
        \State $\Delta A_k \gets (B_k^\top B_k)^{-1} B_k^\top
          (\Delta W_k - \Delta B_k A_k)$
        \State $A_{k+1} \gets A_k + \Delta A_k$, \;
          $B_{k+1} \gets B_k + \Delta B_k$
      \EndFor
    \end{algorithmic}
  \end{algorithm}
  \vspace{-6pt}
\end{wrapfigure}

\paragraph{Algorithm 1.}
Algorithm~\ref{alg:lora-tsd-simple} gives the conceptual form of LoRA-TSD. It
contains the full mathematical update: compute the tangent-projected gradient,
refine it with the tangent spectral step, and map the resulting tangent update
back to LoRA factors through Eqs.~\eqref{eq:lora-delta-b}
and~\eqref{eq:lora-delta-a}. It omits momentum and engineering details, but
this is the form analyzed in Section~\ref{sec:theory}.

\paragraph{Discussion.}
For $\tau=1$, Algorithm~\ref{alg:lora-tsd-simple} recovers the tangent-direction
part of Riemannion~\citep{riemannion}. Table~\ref{tab:tau-ablation} shows that
using $\tau$ from $3$ to $5$ improves average downstream accuracy over
$\tau=1$. LoRA-TSD maps the resulting direction through a LoRA-native factor
retraction rather than the truncated-SVD retraction, which keeps the update in
the factor parametrization and avoids an expensive SVD.
Appendix~\ref{app:retraction} compares the resulting iterates in detail.

\paragraph{Transformation invariance.}
Since the update is defined in weight space, it should not depend on which LoRA
factorization represents the same adapter.
Following LoRA-Rite~\citep{yen2024lorarite}, we call a LoRA optimizer
transformation invariant as follows. Let $(B_k,A_k)$ be a pair of LoRA factors
and let $B'_k = B_kR^{-1}$, $A'_k = RA_k$ for some invertible
$R \in \mathbb{R}^{r \times r}$. If one optimizer step gives updates
$(\Delta B_k,\Delta A_k)$ and $(\Delta B'_k,\Delta A'_k)$, then the optimizer is
transformation invariant if
\begin{equation}
  (B_k+\Delta B_k)(A_k+\Delta A_k)
  = (B'_k+\Delta B'_k)(A'_k+\Delta A'_k)
  := B_kA_k+\Delta W_k .
  \nonumber
\end{equation}
\vspace{-12pt}
\begin{proposition}
  \label{prop:transformation-invariance}
  LoRA-TSD (Algorithm~\ref{alg:lora-tsd-simple}) is transformation invariant.
\end{proposition}

This property is induced by the manifold view. The
tangent projector, the projected gradient $\mathcal{P}_T(G_W^k)$, and the
spectral tangent step $X_\tau$ depend only on the point
$W_k \in W_0+\mathcal{M}_r$, not on the particular LoRA factors used to
represent it. 

We next show how to implement the same update in practice without expensive
operations on $m \times n$ matrices.

\paragraph{Factorization trick.}
The closed form in Eq.~\eqref{eq:lora-tangent-g-closed-form} still appears to
define an $m \times n$ matrix. In practice we never materialize it. Regrouping
the tangent projector of
Eq.~\eqref{eq:lora-tangent-projector} as
\begin{equation}
  \mathcal{P}_T(Z) = P_B Z(I-P_A) + Z P_A
  \label{eq:lora-projector-2term}
\end{equation}
exhibits every tangent matrix as a sum of two rank-$r$ terms, so it has rank at
most $2r$. We therefore store each iterate of
Eq.~\eqref{eq:lora-alt-proj-iteration} in a factored form $X_t = L_t R_t$ with
$L_t \in \mathbb{R}^{m \times 2r}$ and $R_t \in \mathbb{R}^{2r \times n}$.

For the initialization $X_0 = \mathcal{P}_T(G_W)$, applying
Eq.~\eqref{eq:lora-projector-2term} with $Z = G_W$ gives
\begin{equation}
  X_0 = L_0 R_0,
  \label{eq:lora-factorized-init}
\end{equation}
where
\begin{align}
  L_0
  =
  \bigl[\,
    B,\;
    G_B(AA^\top)^{-1}
  \,\bigr],
  \qquad
  R_0
  =
  \begin{bmatrix}
    (B^\top B)^{-1}G_A(I-P_A) \\
    A
  \end{bmatrix} .
  \label{eq:lora-factorized-init-factors}
\end{align}
Here $G_A(I-P_A) = G_A - (G_A A^\top)(AA^\top)^{-1}A$ stays $r \times n$, so no
$m \times n$ product is ever formed.

The matrix sign step is also applied through small matrices. Given
$X_t = L_tR_t$, compute thin QR decompositions
$L_t = Q_L \widehat{R}_L$ and $R_t^\top = Q_R \widehat{R}_R$. Then
\begin{equation}
  \mathrm{msign}(X_t)
  = Q_L\,\mathrm{msign}(\widehat{R}_L\widehat{R}_R^\top)\,Q_R^\top ,
  \label{eq:lora-factorized-msign}
\end{equation}
so the matrix sign reduces to a single $2r \times 2r$ problem. In practice we
compute $\mathrm{msign}(\widehat{R}_L\widehat{R}_R^\top)$ with a Newton--Schulz
iteration rather than an explicit SVD, as in Muon~\citep{muon}. Writing
$\mathrm{msign}(\widehat{R}_L\widehat{R}_R^\top)=UV^\top$, we set
$\widetilde{L}_t = Q_LU$ and $\widetilde{R}_t = V^\top Q_R^\top$.

Finally, we project the msign output back onto the tangent space, again without
forming the $m \times n$ product $\widetilde{L}_t\widetilde{R}_t$. Using the
two-term projector~\eqref{eq:lora-projector-2term} and the small matrices
\begin{equation}
  C_t = B^\top \widetilde{L}_t \in \mathbb{R}^{r \times 2r},
  \qquad
  D_t = \widetilde{R}_t A^\top \in \mathbb{R}^{2r \times r} ,
  \nonumber
\end{equation}
the two projector terms become, with $Z = \widetilde{L}_t\widetilde{R}_t$,
\begin{align}
  P_B Z(I-P_A) &= B(B^\top B)^{-1}(C_t\widetilde{R}_t)(I-P_A), \nonumber \\
  Z P_A &= (\widetilde{L}_t D_t)(AA^\top)^{-1}A , \nonumber
\end{align}
neither of which forms an $m \times n$ matrix. Collecting the shared factors
gives the projection directly in factored form,
$\mathcal{P}_T(\widetilde{L}_t\widetilde{R}_t) = L_{t+1}R_{t+1}$ with
\begin{align}
  L_{t+1}
  =
  \bigl[\,
    B,\;
    \widetilde{L}_tD_t(AA^\top)^{-1}
  \,\bigr],
  \qquad
  R_{t+1}
  =
  \begin{bmatrix}
    (B^\top B)^{-1}C_t\widetilde{R}_t(I-P_A) \\
    A
  \end{bmatrix} .
  \label{eq:lora-factorized-projection}
\end{align}

\begin{wrapfigure}[23]{r}{0.59\textwidth}
  \vspace{-15pt}
  \begin{algorithm}[width=\linewidth]{LoRA-TSD with factorization trick}
    \label{alg:lora-tsd}
    \begin{algorithmic}[1]
    \State $M_A^{0}, M_B^{0} \gets 0, 0$
    \For{$k = 1, 2, \dots$}
      \State $G_A^k \gets \nabla_A \mathcal{L}(W_k)$, \;
        $G_B^k \gets \nabla_B \mathcal{L}(W_k)$
      \State $M_A^k \gets \mu M_A^{k-1} + (1-\mu) G_A^k$
      \State  $M_B^k \gets \mu M_B^{k-1} + (1-\mu) G_B^k$
      \Statex \textbf{Factorized projection onto} \(\mathcal{T}_{W_k}\mathcal{M}_r\)
      \State $L_0 \gets \bigl[B_k,\; M_B^k(A_kA_k^\top)^{-1}\bigr]$
      \State $R_0 \gets \bigl[(B_k^\top B_k)^{-1}M_A^k(I-P_{A_k});\; A_k\bigr]$
      \For{$t = 0, \dots, \tau - 1$}
        \Statex \textbf{Factorized msign}
        \State $L_t = Q_L\widehat{R}_L$, \;
          $R_t^\top = Q_R\widehat{R}_R$
        \State $UV^\top \gets \mathrm{msign}\bigl(\widehat{R}_L\widehat{R}_R^\top\bigr)$
          \Comment{Newton--Schulz}
        \State $\widetilde{L}_t \gets Q_LU$, \;
          $\widetilde{R}_t \gets V^\top Q_R^\top$
        \Statex \textbf{Factorized projection}  
        \State $C_t \gets B_k^\top\widetilde{L}_t$, \;
          $D_t \gets \widetilde{R}_tA_k^\top$
        \State $L_{t+1} \gets \bigl[B_k,\; \widetilde{L}_tD_t(A_kA_k^\top)^{-1}\bigr]$
        \State $R_{t+1} \gets \bigl[(B_k^\top B_k)^{-1}C_t\widetilde{R}_t(I-P_{A_k});\; A_k\bigr]$
      \EndFor
      \Statex \textbf{Retraction}  
      \State $\Delta W_k \gets -\eta L_\tau R_\tau$
      \State $\Delta B_k \gets \Delta W_k A_k^\top (A_k A_k^\top)^{-1}$
      \State $\Delta A_k \gets (B_k^\top B_k)^{-1} B_k^\top
        (\Delta W_k - \Delta B_k A_k)$
      \State $A_{k+1} \gets A_k + \Delta A_k$, \;
        $B_{k+1} \gets B_k + \Delta B_k$
      \EndFor
    \end{algorithmic}
  \end{algorithm}
  \vspace{0pt}
\end{wrapfigure}

This mirrors the block structure of the initialization $L_0, R_0$ in
Eqs.~\eqref{eq:lora-factorized-init}--\eqref{eq:lora-factorized-init-factors}.
Hence the alternating projection loop reuses the Gram inverses $(B^\top B)^{-1}$
and $(AA^\top)^{-1}$ already formed for $X_0$, stores only the thin factors $L_t$
and $R_t$ of size $m \times 2r$ and $2r \times n$ together with the small
$2r \times 2r$ matrix used for the msign step, and never allocates an
$m \times n$ matrix.

\paragraph{Algorithm 2.}
Algorithm~\ref{alg:lora-tsd} is the practical implementation of LoRA-TSD. It
adds momentum to Algorithm~\ref{alg:lora-tsd-simple} and replaces each explicit
tangent matrix by its factored representation $X_t=L_tR_t$ from
Eqs.~\eqref{eq:lora-factorized-init}--\eqref{eq:lora-factorized-projection}.
This factorization is only an implementation device: the underlying tangent
update is unchanged, but all storage and multiplications stay in the thin
factors, avoiding expensive operations on $m \times n$ matrices.


\section{Experiments}
\label{sec:experiments}
\paragraph{Setup.}
We evaluate two models, Llama-3.2-1B-Instruct and Llama-3.1-8B, on
BoolQ~\citep{clark2019boolq}, PIQA~\citep{bisk2020piqa},
SIQA~\citep{sap2019socialiqa}, OBQA~\citep{mihaylov2018openbookqa},
QNLI~\citep{wang2019glue}, and MultiNLI~\citep{williams2018multinli}.
We also evaluate Qwen3-32B~\citep{yang2025qwen3technicalreport} on SIQA and
OBQA.
These benchmarks test natural-language understanding through classification
and multiple-choice questions, covering reading comprehension, physical and
social commonsense, elementary science, and textual entailment. Such
classification-style benchmarks are standard in the LoRA
literature~\citep{yen2024lorarite,tastan2025loft,liu2024dora}.
To isolate optimizer effects, all
methods use Riemannion's locally optimal initialization (LOI), which improves
average accuracy over standard initialization~\citep{riemannion}.
Appendix~\ref{app:experiments} specifies the
adapter placement, preprocessing, and hyperparameters.

\vspace{4pt}
\begin{table*}[h]
\centering
\small
\setlength{\tabcolsep}{4pt}
\resizebox{\textwidth}{!}{%
\begin{tabular}{l rrrrrrrrr | r}
\toprule
Dataset
  & SGD & AdamW & Muon & Riem-SGD & LoRA-Rite
  & Riemannion & LoRA-Pro & LoRA-Muon & \textbf{LoRA-TSD}
  & Muon FT \\
\midrule
BoolQ    & $80.93_{\pm0.21}$ & $80.90_{\pm0.80}$ & $83.07_{\pm0.40}$ & $79.70_{\pm0.35}$ & $83.03_{\pm0.06}$ & $81.73_{\pm0.21}$ & $80.07_{\pm0.49}$ & $83.47_{\pm0.59}$ & $\mathbf{84.63_{\pm0.42}}$ & $83.10_{\pm0.30}$ \\
PIQA     & $76.23_{\pm0.29}$ & $75.77_{\pm0.35}$ & $78.10_{\pm0.50}$ & $64.43_{\pm12.19}$ & $78.80_{\pm0.35}$ & $77.40_{\pm0.52}$ & $74.23_{\pm0.23}$ & $79.17_{\pm0.51}$ & $\mathbf{80.53_{\pm0.25}}$ & $79.20_{\pm0.66}$ \\
SIQA     & $71.97_{\pm0.35}$ & $71.03_{\pm0.32}$ & $73.90_{\pm0.50}$ & $43.80_{\pm2.26}$ & $74.37_{\pm0.32}$ & $72.10_{\pm0.30}$ & $68.53_{\pm0.50}$ & $73.93_{\pm0.12}$ & $\mathbf{75.40_{\pm0.35}}$ & $73.67_{\pm0.46}$ \\
OBQA     & $70.87_{\pm0.50}$ & $72.20_{\pm0.20}$ & $71.87_{\pm0.76}$ & $46.80_{\pm1.06}$ & $\mathbf{74.00_{\pm1.00}}$ & $72.07_{\pm0.81}$ & $71.33_{\pm0.42}$ & $73.73_{\pm1.45}$ & $\mathbf{74.00_{\pm1.20}}$ & $71.93_{\pm0.95}$ \\
QNLI     & $91.03_{\pm0.35}$ & $90.80_{\pm0.44}$ & $92.27_{\pm0.06}$ & $82.30_{\pm4.12}$ & $92.80_{\pm0.26}$ & $91.57_{\pm0.23}$ & $89.30_{\pm0.20}$ & $92.93_{\pm0.21}$ & $\mathbf{93.07_{\pm0.42}}$ & $93.23_{\pm0.25}$ \\
MultiNLI & $53.07_{\pm27.31}$ & $83.33_{\pm0.50}$ & $85.90_{\pm0.53}$ & $74.70_{\pm5.19}$ & $86.17_{\pm0.40}$ & $84.63_{\pm0.31}$ & $83.02_{\pm1.14}$ & $86.37_{\pm0.06}$ & $\mathbf{87.33_{\pm0.06}}$ & $86.63_{\pm0.06}$ \\
\midrule
AVG      & $74.02_{\pm4.42}$ & $79.01_{\pm0.14}$ & $80.85_{\pm0.09}$ & $65.29_{\pm1.39}$ & $81.53_{\pm0.25}$ & $79.92_{\pm0.27}$ & $77.75_{\pm0.19}$ & $81.60_{\pm0.32}$ & $\mathbf{82.49_{\pm0.16}}$ & $81.29_{\pm0.24}$ \\
\bottomrule
\end{tabular}%
}
\caption{Accuracy on commonsense and NLI benchmarks with Llama-3.2-1B-Instruct,
  reported as mean $\pm$ standard deviation over three runs. Bold indicates the
  best LoRA result per dataset.}
\label{tab:main_results}
\end{table*}

\vspace{-8pt}

\paragraph{Baselines.}
We compare against standard factor-space optimizers, SGD, AdamW, and Muon
applied independently to the two LoRA factors. We include three manifold
methods, LoRA-Muon, Riemannian SGD, and Riemannion. LoRA-Rite and LoRA-Pro
cover LoRA-specific optimizer designs. Together, these baselines cover much
of the LoRA optimization design space, including direct factor-space updates,
manifold optimization, gradient preconditioning, and weight-space correction.
Finally, Muon FT fine-tunes the same adapted projections with Muon and
serves as a full-tuning reference.

\vspace{4pt}
\begin{table*}[h]
\centering
\small
\setlength{\tabcolsep}{4pt}
\resizebox{\textwidth}{!}{%
\begin{tabular}{l rrrrrrrrr | r}
\toprule
Dataset
  & SGD & AdamW & Muon & Riem-SGD & LoRA-Rite
  & Riemannion & LoRA-Pro & LoRA-Muon & \textbf{LoRA-TSD}
  & Muon FT \\
\midrule
BoolQ    & 88.6 & 89.6 & 89.1 & 88.8 & 88.9 & 88.5 & 88.3 & 89.8 & \textbf{90.0} & 89.3 \\
PIQA     & 89.1 & 89.2 & 87.9 & 88.3 & 88.4 & 86.9 & 86.3 & 89.2 & \textbf{90.1} & 89.1 \\
SIQA     & 82.2 & 81.8 & 81.2 & 34.1 & 81.4 & 79.4 & 77.9 & 81.8 & \textbf{82.7} & 80.8 \\
OBQA     & 86.0 & 86.8 & 86.0 & 51.4 & 85.6 & 84.0 & 86.8 & 86.6 & \textbf{88.6} & 86.6 \\
QNLI     & 94.8 & 94.3 & 94.8 & 93.7 & 95.0 & 93.5 & 90.2 & 95.1 & \textbf{96.1} & 95.1 \\
MultiNLI & 89.9 & 89.5 & 89.7 & 89.9 & 90.3 & 88.1 & 88.2 & 90.8 & \textbf{90.9} & 90.5 \\
\midrule
AVG      & 88.5 & 88.5 & 88.1 & 74.4 & 88.3 & 86.8 & 86.3 & 88.9 & \textbf{89.7} & 88.5 \\
\bottomrule
\end{tabular}
}
\caption{Accuracy on commonsense and NLI benchmarks with Llama-3.1-8B.
  Bold indicates the best LoRA result per dataset.}
\label{tab:main_results_2}
\end{table*}

\vspace{-8pt}

\paragraph{Results.}
Tables~\ref{tab:main_results}, \ref{tab:main_results_2},
and~\ref{tab:qwen32b-transfer} report accuracy on the three models. Across the
two Llama models, LoRA-TSD is best or tied for best among LoRA optimizers in
all 12 model-dataset settings and achieves the highest average at both model
scales. Notably, it also outperforms Muon Full-FT in 11 of the 12 settings,
despite updating only low-rank adapters. Full parameter access therefore does not consistently
translate into better downstream accuracy than a spectral update constrained
to the LoRA tangent space. LoRA-Muon is the strongest competing LoRA method on
average at both scales, although its average accuracy varies about twice as
much across seeds as LoRA-TSD at the 1B scale. Both SGD baselines are less
stable. Standard SGD has high variance on MultiNLI at the 1B scale. Riemannian
SGD fails to converge reliably on SIQA and OBQA at both scales and is unstable
on several additional 1B tasks. We tuned both with the same protocol as every
other method, described in Appendix~\ref{app:experiments}, and report the
measured results without additional tuning. Table~\ref{tab:qwen32b-transfer}
reports Qwen3-32B results on SIQA and OBQA, the two lowest-accuracy
Llama-3.1-8B benchmarks in Table~\ref{tab:main_results_2}. Qwen3-32B uses the
transferred 8B hyperparameters with minimal tuning, gradient clipping reduced
from $1$ to $0.25$, a frozen FP8-E4M3FN backbone, and FP32 LoRA parameters and
optimizer states.
LoRA-TSD and LoRA-Rite are the two strongest methods on both tasks.

\begingroup
\setlength{\intextsep}{4pt plus 1pt minus 1pt}

\begin{braintablebox}
\centering
\small
\setlength{\tabcolsep}{10pt}
\begin{tabular}{l rrrrr}
\toprule
Dataset & AdamW & Muon & LoRA-Muon & LoRA-Rite & \textbf{LoRA-TSD} \\
\midrule
SIQA & 84.3 & 84.2 & 84.4 & 84.8 & \textbf{85.0} \\
OBQA & 94.6 & 94.4 & 95.0 & 95.6 & \textbf{96.0} \\
\bottomrule
\end{tabular}
\captionof{table}{Accuracy on SIQA and OBQA with Qwen3-32B. Bold indicates the best
  result per dataset.}
\label{tab:qwen32b-transfer}
\end{braintablebox}

\begin{wrapfigure}{r}{0.43\textwidth}
  \vspace{-2pt}
  \centering
\begin{tikzpicture}
\begin{axis}[
    width=\linewidth,
    height=0.54\linewidth,
    xlabel={LoRA rank $r$},
    ylabel={SIQA accuracy},
    xmode=log,
    log basis x=2,
    xtick={1,2,4,8,16,32,64},
    xticklabels={1,2,4,8,16,32,64},
    ymin=77, ymax=84,
    ytick={77,78,79,80,81,82,83,84},
    legend pos=south west,
    legend style={font=\scriptsize, cells={anchor=west}, inner sep=1.5pt,
      row sep=-1pt, nodes={inner sep=1.5pt}},
    legend image post style={scale=0.7},
    grid=major,
    grid style={gray!25},
    tick label style={font=\footnotesize},
    label style={font=\small},
    mark size=2pt,
]
\addplot[HWPrimary, thick, mark=*] coordinates {
    (1,82.14) (2,82.60) (4,82.50) (8,83.47) (16,83.06) (32,82.45) (64,82.29)
};
\addlegendentry{LoRA-TSD}
\addplot[HWSecondary!80!black, thick, mark=square*] coordinates {
    (1,81.78) (2,82.65) (4,82.55) (8,81.73) (16,81.40) (32,81.06) (64,77.84)
};
\addlegendentry{LoRA-Rite}
\addplot[HWAccent!70!black, thick, mark=triangle*] coordinates {
    (1,80.55) (2,80.71) (4,80.86) (8,81.53) (16,81.20) (32,81.01) (64,81.06)
};
\addlegendentry{Muon}
\addplot[gray, thick, dashed, mark=diamond*] coordinates {
    (1,80.96) (2,81.78) (4,80.86) (8,81.73) (16,81.80) (32,81.83) (64,81.83)
};
\addlegendentry{AdamW}
\end{axis}
\end{tikzpicture}
  \captionsetup{font=normalsize}
  \caption{Rank ablation on SIQA (Llama-3.1-8B). LoRA-TSD maintains high
    accuracy across ranks.}
  \label{fig:rank_ablation}
\end{wrapfigure}

\paragraph{Number of inner iterations.}
Table~\ref{tab:tau-ablation} shows a trade-off between quality and training
cost. With $\tau=1$, LoRA-TSD is $0.92$ average accuracy points below
$\tau=3$, but runs faster than Riemannion because its factor-induced retraction
is cheaper than truncated SVD (Table~\ref{tab:retraction-time}). Increasing
$\tau$ to $3$ gives the best average accuracy at the cost of two additional
inner projections per training step. Increasing $\tau$ further to $4$ or $5$
provides no noticeable performance gain.

\paragraph{Rank ablation.}
Figure~\ref{fig:rank_ablation} shows that LoRA-TSD maintains the most stable
high-accuracy performance across ranks on SIQA. Appendix~\ref{app:experiments}
provides a deeper comparison on BoolQ and SIQA.

\endgroup

\begin{braintablebox}
\centering
\small
\setlength{\tabcolsep}{7pt}
\begin{tabular}{r rrr r}
\toprule
$\tau$ & BoolQ & PIQA & SIQA & AVG \\
\midrule
1 & 84.16 & 79.05 & 75.59 & 79.60 \\
2 & 84.40 & 79.65 & 74.51 & 79.52 \\
3 & 84.89 & \textbf{80.20} & \textbf{76.46} & \textbf{80.52} \\
4 & 84.68 & 80.14 & 75.90 & 80.24 \\
5 & \textbf{85.08} & 79.87 & \textbf{76.46} & 80.47 \\
\bottomrule
\end{tabular}
\captionof{table}{Accuracy with Llama-3.2-1B-Instruct for different numbers $\tau$ of
  alternating-projection steps. Bold marks the best result in each column.}
\label{tab:tau-ablation}
\end{braintablebox}


\vspace{-12pt}
\section{Theory}
\label{sec:theory}
\subsection{Convergence guarantees}
\label{subsec:convergence-guarantees}
To our knowledge, LoRA-TSD is the first LoRA-specific optimizer with a
stationary-point convergence guarantee for the nonconvex fixed-rank LoRA
problem, expressed through the tangent-projected gradient. Existing convergence
analyses of LoRA study standard gradient methods on the factorized objective.
\citet{mu2026on} establish a convergence rate for deterministic LoRA gradient
descent, while \citet{wang2026on} derive guarantees for deterministic and
stochastic variants. LoRA-RITE~\citep{yen2024lorarite} provides a different
result: an online-optimization bound on an averaged first-order term,
strengthened under an additional smooth-change assumption on its
preconditioners. Our analysis instead directly bounds the tangent-projected
gradient, the natural stationarity measure for the fixed-rank problem. We next
state the assumptions and define this measure formally.

\begin{assumption}[Spectral smoothness and bounded gradients]
The gradient $\nabla\mathcal{L}$ is Lipschitz continuous with respect to
the spectral norm:
\[
  \|\nabla\mathcal{L}(W) - \nabla\mathcal{L}(W')\|_* \le M\|W - W'\|_2 .
\]
Moreover, $\|\nabla\mathcal{L}(W)\|_2 \le G_{\max}$ along the trajectory.
\end{assumption}

\begin{assumption}[Lower-bounded objective]
The objective is bounded from below: $\mathcal{L} \ge \mathcal{L}_*$.
\end{assumption}

\begin{assumption}[Uniform factor nondegeneracy]
The LoRA factors remain full rank along the trajectory: there exists
$\delta > 0$ such that
\[
  \sigma_r(A_k),\, \sigma_r(B_k) \ge \delta > 0 \quad \forall\, k.
\]
\end{assumption}

\begin{assumption}[Uniform descent alignment]
For some $c_0 > 0$,
\[
  \langle \mathcal{P}_{T_k}(G_W^k), X_\tau \rangle
  \ge c_0 \|\mathcal{P}_{T_k}(G_W^k)\|_* ,
\]
where $X_\tau$ is the alternating-projection output.
\end{assumption}

\paragraph{Discussion of assumptions.}
Assumptions \textbf{(A1)} and \textbf{(A2)}, together with the step-size bound,
are standard in non-convex
optimization~\citep{ghadimi2013stochastic,bottou2018optimization}. Assumption
\textbf{(A3)} keeps the LoRA factors away from rank deficiency. This condition
must account for the gauge ambiguity because the same residual $B_kA_k$ can be
represented with one factor arbitrarily small and the other arbitrarily large.
We remove this artificial scale imbalance by periodically applying a gauge
transformation that preserves $B_kA_k$, and hence $W_k$, while rebalancing the
factor norms:
\begin{equation}
  c_k = \sqrt{\frac{\|B_k\|_F}{\|A_k\|_F}},
  \qquad A_k \leftarrow c_k A_k,
  \qquad B_k \leftarrow c_k^{-1} B_k .
  \label{eq:rebalancing}
  \nonumber
\end{equation}
We rescale the momentum buffers as $M_A \leftarrow c_k^{-1}M_A$ and
$M_B \leftarrow c_kM_B$ to preserve the induced update direction.
Figure~\ref{fig:sigma_r} shows that the smallest singular values
$\sigma_r(A_k)$ and $\sigma_r(B_k)$ remain above a positive $\delta$ throughout
training, as required by \textbf{(A3)}. The last condition, \textbf{(A4)}, requires the approximate tangent oracle to
retain a uniform descent component. For $\tau=1$, it holds with $c_0=1$ because
$\mathcal{P}_T=\mathcal{P}_T^*$ and
$\mathcal{P}_T(g)=g$,
\begin{align*}
  \langle g, X_1 \rangle
  = \langle g, \mathcal{P}_T(\mathrm{msign}(g)) \rangle
  = \langle \mathcal{P}_T(g), \mathrm{msign}(g) \rangle
  = \langle g, \mathrm{msign}(g) \rangle
  = \|g\|_* .
\end{align*}
For $\tau>1$, \textbf{(A4)} can be enforced by returning
\[
  \widehat X_\tau
  \in \operatorname*{arg\,max}_{1\le t\le\tau}\langle g,X_t\rangle .
\]
Since $X_1$ is included, this safeguard gives $c_0=1$. It does not noticeably
improve downstream performance, so we use $X_\tau$ and
verify \textbf{(A4)} empirically in Figure~\ref{fig:a4_cos}.

\newif\ifwideassumptionfigure
\wideassumptionfiguretrue
\ifwideassumptionfigure
  \begin{figure}[t]
    \centering
    \captionsetup{font=normalsize}
    \begin{subfigure}[t]{0.485\textwidth}
      \centering
\begin{tikzpicture}
\begin{axis}[
    width=0.92\linewidth, height=0.55\linewidth,
    xlabel={training step $k$},
    ylabel={$\min_\ell \sigma_r$},
    xmin=250, xmax=3000,
    ymin=0, ymax=0.13,
    xtick={250,500,1000,1500,2000,2500,3000},
    ytick={0,0.03,0.06,0.09,0.12},
    yticklabels={0,0.03,0.06,0.09,0.12},
    grid=major, grid style={gray!25},
    tick label style={font=\small}, label style={font=\small},
    legend pos=south east,
    legend style={font=\scriptsize, cells={anchor=west}, inner sep=1.5pt, row sep=-1pt},
]
\addplot[HWPrimary, thick, mark=*, mark size=0.8pt] coordinates {(251,3.678032e-02) (276,3.928674e-02) (301,4.166880e-02) (326,4.371792e-02) (351,4.492046e-02) (376,4.547013e-02) (401,4.602891e-02) (426,4.650519e-02) (451,4.696169e-02) (476,4.741003e-02) (501,4.161361e-02) (526,4.195654e-02) (551,4.227622e-02) (576,4.259074e-02) (601,4.288288e-02) (626,4.317441e-02) (651,4.344507e-02) (676,4.371039e-02) (701,4.396308e-02) (726,4.419487e-02) (751,4.317682e-02) (776,4.339064e-02) (801,4.360814e-02) (826,4.381288e-02) (851,4.401708e-02) (876,4.420240e-02) (901,4.439357e-02) (926,4.456717e-02) (951,4.474041e-02) (976,4.490259e-02) (1001,4.444576e-02) (1026,4.459822e-02) (1051,4.474092e-02) (1076,4.488202e-02) (1101,4.501350e-02) (1126,4.514527e-02) (1151,4.527238e-02) (1176,4.539738e-02) (1201,4.551546e-02) (1226,4.563063e-02) (1251,4.540843e-02) (1276,4.551300e-02) (1301,4.561284e-02) (1326,4.571224e-02) (1351,4.580429e-02) (1376,4.589474e-02) (1401,4.598290e-02) (1426,4.606984e-02) (1451,4.615213e-02) (1476,4.623501e-02) (1501,4.607995e-02) (1526,4.615276e-02) (1551,4.622294e-02) (1576,4.629213e-02) (1601,4.636202e-02) (1626,4.642565e-02) (1651,4.648423e-02) (1676,4.654402e-02) (1701,4.659889e-02) (1726,4.665198e-02) (1751,4.657568e-02) (1776,4.662289e-02) (1801,4.667160e-02) (1826,4.671757e-02) (1851,4.676051e-02) (1876,4.680089e-02) (1901,4.683981e-02) (1926,4.687649e-02) (1951,4.691093e-02) (1976,4.694424e-02) (2001,4.688404e-02) (2026,4.691399e-02) (2051,4.694283e-02) (2076,4.697026e-02) (2101,4.699653e-02) (2126,4.702058e-02) (2151,4.704479e-02) (2176,4.706559e-02) (2201,4.708618e-02) (2226,4.710575e-02) (2251,4.708093e-02) (2276,4.709788e-02) (2301,4.711373e-02) (2326,4.712908e-02) (2351,4.714284e-02) (2376,4.715525e-02) (2401,4.716659e-02) (2426,4.717705e-02) (2451,4.718688e-02) (2476,4.719600e-02) (2501,4.717372e-02) (2526,4.718112e-02) (2551,4.718798e-02) (2576,4.719428e-02) (2601,4.719964e-02) (2626,4.720430e-02) (2651,4.720816e-02) (2676,4.721179e-02) (2701,4.721493e-02) (2726,4.721751e-02) (2751,4.720293e-02) (2776,4.720475e-02) (2801,4.720613e-02) (2826,4.720720e-02) (2851,4.720804e-02) (2876,4.720866e-02) (2901,4.720908e-02) (2926,4.720934e-02) (2951,4.720948e-02) (2976,4.720954e-02)};
\addlegendentry{$\sigma_r(A_k)$}
\addplot[HWSecondary!80!black, thick, mark=square*, mark size=0.8pt] coordinates {(251,4.471150e-02) (276,4.805045e-02) (301,5.111825e-02) (326,5.394792e-02) (351,5.626781e-02) (376,5.826166e-02) (401,6.022089e-02) (426,6.186109e-02) (451,6.329467e-02) (476,6.458308e-02) (501,7.567568e-02) (526,7.695246e-02) (551,7.807636e-02) (576,7.921745e-02) (601,8.043171e-02) (626,8.150914e-02) (651,8.253291e-02) (676,8.347933e-02) (701,8.444596e-02) (726,8.526360e-02) (751,8.856214e-02) (776,8.935172e-02) (801,9.008878e-02) (826,9.076120e-02) (851,9.142306e-02) (876,9.201965e-02) (901,9.270095e-02) (926,9.327639e-02) (951,9.389823e-02) (976,9.447765e-02) (1001,9.634112e-02) (1026,9.685105e-02) (1051,9.736886e-02) (1076,9.779619e-02) (1101,9.826240e-02) (1126,9.869048e-02) (1151,9.910443e-02) (1176,9.948806e-02) (1201,9.988174e-02) (1226,1.002605e-01) (1251,1.013432e-01) (1276,1.016774e-01) (1301,1.020140e-01) (1326,1.023325e-01) (1351,1.026369e-01) (1376,1.029419e-01) (1401,1.032025e-01) (1426,1.034822e-01) (1451,1.037529e-01) (1476,1.040015e-01) (1501,1.047875e-01) (1526,1.050079e-01) (1551,1.052344e-01) (1576,1.054673e-01) (1601,1.056919e-01) (1626,1.058841e-01) (1651,1.060905e-01) (1676,1.062845e-01) (1701,1.064688e-01) (1726,1.066290e-01) (1751,1.070814e-01) (1776,1.072300e-01) (1801,1.073766e-01) (1826,1.075102e-01) (1851,1.076397e-01) (1876,1.077785e-01) (1901,1.078999e-01) (1926,1.080076e-01) (1951,1.081220e-01) (1976,1.082158e-01) (2001,1.085333e-01) (2026,1.086235e-01) (2051,1.087143e-01) (2076,1.087997e-01) (2101,1.088739e-01) (2126,1.089481e-01) (2151,1.090115e-01) (2176,1.090733e-01) (2201,1.091419e-01) (2226,1.091934e-01) (2251,1.093525e-01) (2276,1.094040e-01) (2301,1.094495e-01) (2326,1.094893e-01) (2351,1.095290e-01) (2376,1.095564e-01) (2401,1.095943e-01) (2426,1.096246e-01) (2451,1.096572e-01) (2476,1.096829e-01) (2501,1.097769e-01) (2526,1.098047e-01) (2551,1.098263e-01) (2576,1.098434e-01) (2601,1.098580e-01) (2626,1.098714e-01) (2651,1.098767e-01) (2676,1.098846e-01) (2701,1.098977e-01) (2726,1.099040e-01) (2751,1.099468e-01) (2776,1.099532e-01) (2801,1.099570e-01) (2826,1.099607e-01) (2851,1.099649e-01) (2876,1.099676e-01) (2901,1.099693e-01) (2926,1.099697e-01) (2951,1.099702e-01) (2976,1.099706e-01)};
\addlegendentry{$\sigma_r(B_k)$}
\addplot[gray, dashed, thick] coordinates {(250,0.030) (3000,0.030)};
\addlegendentry{$\delta$}
\end{axis}
\end{tikzpicture}
      \caption{Minimum factor singular values across adapted layers. Both remain
        above $\delta$.}
      \label{fig:sigma_r}
    \end{subfigure}
    \hfill
    \begin{subfigure}[t]{0.485\textwidth}
      \centering
\begin{tikzpicture}
\begin{axis}[
    width=0.92\linewidth,
    height=0.55\linewidth,
    xlabel={training step $k$},
    ylabel={$\langle g_k,X_\tau\rangle_F/\|g_k\|_*$},
    ylabel style={font=\small},
    xmin=0, xmax=2000,
    ymin=0.5, ymax=1,
    xtick={0,500,1000,1500,2000},
    ytick={0.5,0.6,0.7,0.8,0.9,1.0},
    grid=major, grid style={gray!25},
    tick label style={font=\small},
    label style={font=\small},
    legend style={
      font=\tiny,
      at={(0.98,0.03)},
      anchor=south east,
      draw=black,
      fill=white,
      fill opacity=0.9,
      text opacity=1,
      row sep=-1pt,
    },
]
\addplot[HWPrimary, thick] coordinates {
  (50,0.6322) (100,0.7995) (150,0.8240) (200,0.8225) (250,0.8579)
  (300,0.8736) (350,0.8715) (400,0.8629) (450,0.8495) (500,0.8475)
  (550,0.8677) (600,0.8391) (650,0.8633) (700,0.8869) (750,0.8882)
  (800,0.9263) (850,0.8920) (900,0.9008) (950,0.8861) (1000,0.8842)
  (1050,0.8794) (1100,0.9041) (1150,0.8818) (1200,0.9103) (1250,0.8893)
  (1300,0.9042) (1350,0.9161) (1400,0.8911) (1450,0.8750) (1500,0.8718)
  (1550,0.8756) (1600,0.8643) (1650,0.8749) (1700,0.8996) (1750,0.8907)
  (1800,0.9392) (1850,0.9104) (1900,0.8672) (1950,0.8997) (2000,0.8970)
};
\addlegendentry{\texttt{q\_proj}}
\addplot[HWSecondary!80!black, thick] coordinates {
  (50,0.6351) (100,0.8073) (150,0.8323) (200,0.8488) (250,0.8570)
  (300,0.8805) (350,0.8655) (400,0.8804) (450,0.8399) (500,0.8672)
  (550,0.9024) (600,0.8907) (650,0.8661) (700,0.8747) (750,0.8469)
  (800,0.8800) (850,0.9104) (900,0.8915) (950,0.9034) (1000,0.9084)
  (1050,0.8578) (1100,0.8971) (1150,0.8700) (1200,0.8872) (1250,0.8585)
  (1300,0.8671) (1350,0.9034) (1400,0.8499) (1450,0.8275) (1500,0.8575)
  (1550,0.8619) (1600,0.8596) (1650,0.9127) (1700,0.8987) (1750,0.8715)
  (1800,0.8678) (1850,0.8593) (1900,0.9129) (1950,0.8905) (2000,0.9190)
};
\addlegendentry{\texttt{k\_proj}}
\addplot[HWAccent!70!black, thick] coordinates {
  (50,0.6050) (100,0.6348) (150,0.7465) (200,0.7570) (250,0.7920)
  (300,0.8775) (350,0.8571) (400,0.8885) (450,0.8898) (500,0.8950)
  (550,0.8837) (600,0.8835) (650,0.8684) (700,0.8748) (750,0.8847)
  (800,0.9074) (850,0.8944) (900,0.8931) (950,0.8947) (1000,0.9006)
  (1050,0.8682) (1100,0.8941) (1150,0.8755) (1200,0.8954) (1250,0.9248)
  (1300,0.8836) (1350,0.8535) (1400,0.8690) (1450,0.9280) (1500,0.8800)
  (1550,0.8933) (1600,0.9066) (1650,0.9501) (1700,0.8699) (1750,0.8853)
  (1800,0.9025) (1850,0.8614) (1900,0.9325) (1950,0.8732) (2000,0.8870)
};
\addlegendentry{\texttt{v\_proj}}
\addplot[black, thick, dashed] coordinates {
  (50,0.5794) (100,0.6076) (150,0.6990) (200,0.7054) (250,0.7315)
  (300,0.7925) (350,0.8059) (400,0.7803) (450,0.7640) (500,0.7881)
  (550,0.7663) (600,0.7477) (650,0.7945) (700,0.8030) (750,0.7970)
  (800,0.7757) (850,0.8121) (900,0.7603) (950,0.8146) (1000,0.8210)
  (1050,0.8317) (1100,0.8009) (1150,0.8177) (1200,0.8198) (1250,0.8444)
  (1300,0.8326) (1350,0.7749) (1400,0.8361) (1450,0.8098) (1500,0.8095)
  (1550,0.8015) (1600,0.8090) (1650,0.8153) (1700,0.8250) (1750,0.8142)
  (1800,0.8010) (1850,0.8220) (1900,0.8116) (1950,0.8251) (2000,0.8207)
};
\addlegendentry{minimum}
\end{axis}
\end{tikzpicture}
      \caption{Descent ratio for selected projection types with $\tau=3$, where
        $g_k$ denotes $\mathcal{P}_{T_k}(G_W^k)$. The dashed curve is the
        minimum across all 96 adapted matrices.}
      \label{fig:a4_cos}
    \end{subfigure}
    \caption{Empirical checks of Assumptions \textbf{(A3)} and \textbf{(A4)}
      during LoRA-TSD training on BoolQ with Llama-3.1-8B.}
    \label{fig:assumption_checks}
  \end{figure}
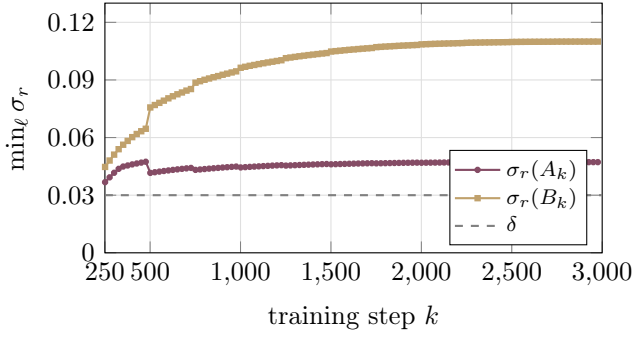
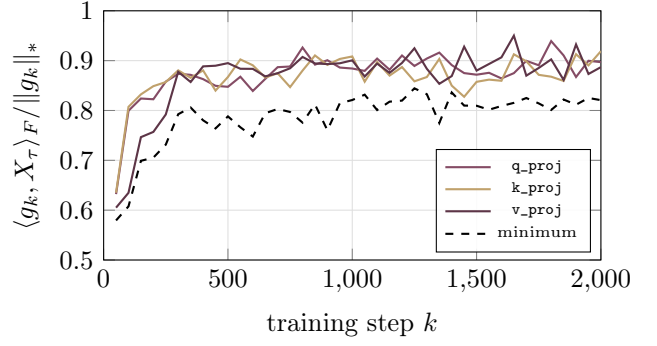
\else
  \begin{figure}[H]
    \centering
    \captionsetup{font=normalsize}
    \begin{subfigure}[t]{\linewidth}
      \centering
      \makebox[\linewidth][c]{\hspace*{-0.01\linewidth}%
\begin{tikzpicture}
\begin{axis}[
    width=0.92\linewidth, height=0.55\linewidth,
    xlabel={training step $k$},
    ylabel={$\min_\ell \sigma_r$},
    xmin=250, xmax=3000,
    ymin=0, ymax=0.13,
    xtick={250,500,1000,1500,2000,2500,3000},
    ytick={0,0.03,0.06,0.09,0.12},
    yticklabels={0,0.03,0.06,0.09,0.12},
    grid=major, grid style={gray!25},
    tick label style={font=\small}, label style={font=\small},
    legend pos=south east,
    legend style={font=\scriptsize, cells={anchor=west}, inner sep=1.5pt, row sep=-1pt},
]
\addplot[HWPrimary, thick, mark=*, mark size=0.8pt] coordinates {(251,3.678032e-02) (276,3.928674e-02) (301,4.166880e-02) (326,4.371792e-02) (351,4.492046e-02) (376,4.547013e-02) (401,4.602891e-02) (426,4.650519e-02) (451,4.696169e-02) (476,4.741003e-02) (501,4.161361e-02) (526,4.195654e-02) (551,4.227622e-02) (576,4.259074e-02) (601,4.288288e-02) (626,4.317441e-02) (651,4.344507e-02) (676,4.371039e-02) (701,4.396308e-02) (726,4.419487e-02) (751,4.317682e-02) (776,4.339064e-02) (801,4.360814e-02) (826,4.381288e-02) (851,4.401708e-02) (876,4.420240e-02) (901,4.439357e-02) (926,4.456717e-02) (951,4.474041e-02) (976,4.490259e-02) (1001,4.444576e-02) (1026,4.459822e-02) (1051,4.474092e-02) (1076,4.488202e-02) (1101,4.501350e-02) (1126,4.514527e-02) (1151,4.527238e-02) (1176,4.539738e-02) (1201,4.551546e-02) (1226,4.563063e-02) (1251,4.540843e-02) (1276,4.551300e-02) (1301,4.561284e-02) (1326,4.571224e-02) (1351,4.580429e-02) (1376,4.589474e-02) (1401,4.598290e-02) (1426,4.606984e-02) (1451,4.615213e-02) (1476,4.623501e-02) (1501,4.607995e-02) (1526,4.615276e-02) (1551,4.622294e-02) (1576,4.629213e-02) (1601,4.636202e-02) (1626,4.642565e-02) (1651,4.648423e-02) (1676,4.654402e-02) (1701,4.659889e-02) (1726,4.665198e-02) (1751,4.657568e-02) (1776,4.662289e-02) (1801,4.667160e-02) (1826,4.671757e-02) (1851,4.676051e-02) (1876,4.680089e-02) (1901,4.683981e-02) (1926,4.687649e-02) (1951,4.691093e-02) (1976,4.694424e-02) (2001,4.688404e-02) (2026,4.691399e-02) (2051,4.694283e-02) (2076,4.697026e-02) (2101,4.699653e-02) (2126,4.702058e-02) (2151,4.704479e-02) (2176,4.706559e-02) (2201,4.708618e-02) (2226,4.710575e-02) (2251,4.708093e-02) (2276,4.709788e-02) (2301,4.711373e-02) (2326,4.712908e-02) (2351,4.714284e-02) (2376,4.715525e-02) (2401,4.716659e-02) (2426,4.717705e-02) (2451,4.718688e-02) (2476,4.719600e-02) (2501,4.717372e-02) (2526,4.718112e-02) (2551,4.718798e-02) (2576,4.719428e-02) (2601,4.719964e-02) (2626,4.720430e-02) (2651,4.720816e-02) (2676,4.721179e-02) (2701,4.721493e-02) (2726,4.721751e-02) (2751,4.720293e-02) (2776,4.720475e-02) (2801,4.720613e-02) (2826,4.720720e-02) (2851,4.720804e-02) (2876,4.720866e-02) (2901,4.720908e-02) (2926,4.720934e-02) (2951,4.720948e-02) (2976,4.720954e-02)};
\addlegendentry{$\sigma_r(A_k)$}
\addplot[HWSecondary!80!black, thick, mark=square*, mark size=0.8pt] coordinates {(251,4.471150e-02) (276,4.805045e-02) (301,5.111825e-02) (326,5.394792e-02) (351,5.626781e-02) (376,5.826166e-02) (401,6.022089e-02) (426,6.186109e-02) (451,6.329467e-02) (476,6.458308e-02) (501,7.567568e-02) (526,7.695246e-02) (551,7.807636e-02) (576,7.921745e-02) (601,8.043171e-02) (626,8.150914e-02) (651,8.253291e-02) (676,8.347933e-02) (701,8.444596e-02) (726,8.526360e-02) (751,8.856214e-02) (776,8.935172e-02) (801,9.008878e-02) (826,9.076120e-02) (851,9.142306e-02) (876,9.201965e-02) (901,9.270095e-02) (926,9.327639e-02) (951,9.389823e-02) (976,9.447765e-02) (1001,9.634112e-02) (1026,9.685105e-02) (1051,9.736886e-02) (1076,9.779619e-02) (1101,9.826240e-02) (1126,9.869048e-02) (1151,9.910443e-02) (1176,9.948806e-02) (1201,9.988174e-02) (1226,1.002605e-01) (1251,1.013432e-01) (1276,1.016774e-01) (1301,1.020140e-01) (1326,1.023325e-01) (1351,1.026369e-01) (1376,1.029419e-01) (1401,1.032025e-01) (1426,1.034822e-01) (1451,1.037529e-01) (1476,1.040015e-01) (1501,1.047875e-01) (1526,1.050079e-01) (1551,1.052344e-01) (1576,1.054673e-01) (1601,1.056919e-01) (1626,1.058841e-01) (1651,1.060905e-01) (1676,1.062845e-01) (1701,1.064688e-01) (1726,1.066290e-01) (1751,1.070814e-01) (1776,1.072300e-01) (1801,1.073766e-01) (1826,1.075102e-01) (1851,1.076397e-01) (1876,1.077785e-01) (1901,1.078999e-01) (1926,1.080076e-01) (1951,1.081220e-01) (1976,1.082158e-01) (2001,1.085333e-01) (2026,1.086235e-01) (2051,1.087143e-01) (2076,1.087997e-01) (2101,1.088739e-01) (2126,1.089481e-01) (2151,1.090115e-01) (2176,1.090733e-01) (2201,1.091419e-01) (2226,1.091934e-01) (2251,1.093525e-01) (2276,1.094040e-01) (2301,1.094495e-01) (2326,1.094893e-01) (2351,1.095290e-01) (2376,1.095564e-01) (2401,1.095943e-01) (2426,1.096246e-01) (2451,1.096572e-01) (2476,1.096829e-01) (2501,1.097769e-01) (2526,1.098047e-01) (2551,1.098263e-01) (2576,1.098434e-01) (2601,1.098580e-01) (2626,1.098714e-01) (2651,1.098767e-01) (2676,1.098846e-01) (2701,1.098977e-01) (2726,1.099040e-01) (2751,1.099468e-01) (2776,1.099532e-01) (2801,1.099570e-01) (2826,1.099607e-01) (2851,1.099649e-01) (2876,1.099676e-01) (2901,1.099693e-01) (2926,1.099697e-01) (2951,1.099702e-01) (2976,1.099706e-01)};
\addlegendentry{$\sigma_r(B_k)$}
\addplot[gray, dashed, thick] coordinates {(250,0.030) (3000,0.030)};
\addlegendentry{$\delta$}
\end{axis}
\end{tikzpicture}\hspace*{0.01\linewidth}}
      \caption{Minimum factor singular values across adapted layers. Both remain
        above $\delta$.}
      \label{fig:sigma_r}
    \end{subfigure}
    \medskip
    \begin{subfigure}[t]{\linewidth}
      \centering
\begin{tikzpicture}
\begin{axis}[
    width=0.92\linewidth,
    height=0.55\linewidth,
    xlabel={training step $k$},
    ylabel={$\langle g_k,X_\tau\rangle_F/\|g_k\|_*$},
    ylabel style={font=\small},
    xmin=0, xmax=2000,
    ymin=0.5, ymax=1,
    xtick={0,500,1000,1500,2000},
    ytick={0.5,0.6,0.7,0.8,0.9,1.0},
    grid=major, grid style={gray!25},
    tick label style={font=\small},
    label style={font=\small},
    legend style={
      font=\tiny,
      at={(0.98,0.03)},
      anchor=south east,
      draw=black,
      fill=white,
      fill opacity=0.9,
      text opacity=1,
      row sep=-1pt,
    },
]
\addplot[HWPrimary, thick] coordinates {
  (50,0.6322) (100,0.7995) (150,0.8240) (200,0.8225) (250,0.8579)
  (300,0.8736) (350,0.8715) (400,0.8629) (450,0.8495) (500,0.8475)
  (550,0.8677) (600,0.8391) (650,0.8633) (700,0.8869) (750,0.8882)
  (800,0.9263) (850,0.8920) (900,0.9008) (950,0.8861) (1000,0.8842)
  (1050,0.8794) (1100,0.9041) (1150,0.8818) (1200,0.9103) (1250,0.8893)
  (1300,0.9042) (1350,0.9161) (1400,0.8911) (1450,0.8750) (1500,0.8718)
  (1550,0.8756) (1600,0.8643) (1650,0.8749) (1700,0.8996) (1750,0.8907)
  (1800,0.9392) (1850,0.9104) (1900,0.8672) (1950,0.8997) (2000,0.8970)
};
\addlegendentry{\texttt{q\_proj}}
\addplot[HWSecondary!80!black, thick] coordinates {
  (50,0.6351) (100,0.8073) (150,0.8323) (200,0.8488) (250,0.8570)
  (300,0.8805) (350,0.8655) (400,0.8804) (450,0.8399) (500,0.8672)
  (550,0.9024) (600,0.8907) (650,0.8661) (700,0.8747) (750,0.8469)
  (800,0.8800) (850,0.9104) (900,0.8915) (950,0.9034) (1000,0.9084)
  (1050,0.8578) (1100,0.8971) (1150,0.8700) (1200,0.8872) (1250,0.8585)
  (1300,0.8671) (1350,0.9034) (1400,0.8499) (1450,0.8275) (1500,0.8575)
  (1550,0.8619) (1600,0.8596) (1650,0.9127) (1700,0.8987) (1750,0.8715)
  (1800,0.8678) (1850,0.8593) (1900,0.9129) (1950,0.8905) (2000,0.9190)
};
\addlegendentry{\texttt{k\_proj}}
\addplot[HWAccent!70!black, thick] coordinates {
  (50,0.6050) (100,0.6348) (150,0.7465) (200,0.7570) (250,0.7920)
  (300,0.8775) (350,0.8571) (400,0.8885) (450,0.8898) (500,0.8950)
  (550,0.8837) (600,0.8835) (650,0.8684) (700,0.8748) (750,0.8847)
  (800,0.9074) (850,0.8944) (900,0.8931) (950,0.8947) (1000,0.9006)
  (1050,0.8682) (1100,0.8941) (1150,0.8755) (1200,0.8954) (1250,0.9248)
  (1300,0.8836) (1350,0.8535) (1400,0.8690) (1450,0.9280) (1500,0.8800)
  (1550,0.8933) (1600,0.9066) (1650,0.9501) (1700,0.8699) (1750,0.8853)
  (1800,0.9025) (1850,0.8614) (1900,0.9325) (1950,0.8732) (2000,0.8870)
};
\addlegendentry{\texttt{v\_proj}}
\addplot[black, thick, dashed] coordinates {
  (50,0.5794) (100,0.6076) (150,0.6990) (200,0.7054) (250,0.7315)
  (300,0.7925) (350,0.8059) (400,0.7803) (450,0.7640) (500,0.7881)
  (550,0.7663) (600,0.7477) (650,0.7945) (700,0.8030) (750,0.7970)
  (800,0.7757) (850,0.8121) (900,0.7603) (950,0.8146) (1000,0.8210)
  (1050,0.8317) (1100,0.8009) (1150,0.8177) (1200,0.8198) (1250,0.8444)
  (1300,0.8326) (1350,0.7749) (1400,0.8361) (1450,0.8098) (1500,0.8095)
  (1550,0.8015) (1600,0.8090) (1650,0.8153) (1700,0.8250) (1750,0.8142)
  (1800,0.8010) (1850,0.8220) (1900,0.8116) (1950,0.8251) (2000,0.8207)
};
\addlegendentry{minimum}
\end{axis}
\end{tikzpicture}
      \caption{Descent ratio for selected projection types with $\tau=3$, where
        $g_k$ denotes $\mathcal{P}_{T_k}(G_W^k)$. The dashed curve is the
        minimum across all 96 adapted matrices.}
      \label{fig:a4_cos}
    \end{subfigure}
    \caption{Empirical checks of Assumptions \textbf{(A3)} and \textbf{(A4)}
      during LoRA-TSD training on BoolQ with Llama-3.1-8B.}
    \label{fig:assumption_checks}
  \end{figure}
\fi

\paragraph{Stationarity in the LoRA setting.}
Because $W_k$ is confined to the affine shift of the manifold $\mathcal{M}_r$, the
relevant notion of a stationary point is not $G_W^k = 0$, which is generally
unattainable under a rank constraint, but the vanishing of the gradient's
component tangent to $\mathcal{M}_r$, i.e.\ $\mathcal{P}_T(G_W^k) = 0$. This
tangential stationarity condition is the criterion used in the convergence
bounds below.

\paragraph{Convergence rate.}
For LoRA-TSD, which operates in the spectral-norm geometry matching the
Muon update, we have the following.
\begin{theorem}[Convergence of LoRA-TSD]
  \label{thm:lora-tsd-convergence}
  Let $\{W_k\}$ be the iterates produced by Algorithm~\ref{alg:lora-tsd-simple}
  (the momentum-free variant, Algorithm~\ref{alg:lora-tsd} differs only in
  the addition of momentum).
  Under assumptions \textbf{(A1)}--\textbf{(A4)} and step sizes
  $0 < \eta_k \le \min\!\bigl\{1,\,\delta^2/(2\sqrt{2})\bigr\}$, for any $N \ge 1$:
  \begin{enumerate}
    \item \emph{Constant step size.} If $\eta_k \equiv \eta$, then
    \begin{equation}
      \min_{0 \le k < N} \|\mathcal{P}_{T_k}(G_W^k)\|_*
      \le \frac{\mathcal{L}(W_0) - \mathcal{L}_*}{c_0 \eta N}
      + \frac{C}{c_0}\eta .
      \label{eq:lora-tsd-rate} \nonumber
    \end{equation}
    and in particular, for
    $\eta = \sqrt{(\mathcal{L}(W_0) - \mathcal{L}_*)/(CN)}$,
    \begin{equation}
      \min_{0 \le k < N} \|\mathcal{P}_{T_k}(G_W^k)\|_*
      \le \frac{2}{c_0}
      \sqrt{\frac{C(\mathcal{L}(W_0) - \mathcal{L}_*)}{N}}
      = O\bigl(N^{-1/2}\bigr) .
      \label{eq:lora-tsd-rate-tuned} \nonumber
    \end{equation}
    \item \emph{Diminishing step size.} If
    $\sum_k \eta_k = \infty$ and $\sum_k \eta_k^2 < \infty$, then
    $\sum_k \eta_k \|\mathcal{P}_{T_k}(G_W^k)\|_* < \infty$, and consequently
    $\lim_{k\to\infty} \|\mathcal{P}_{T_k}(G_W^k)\|_* = 0$.
  \end{enumerate}
\end{theorem}

\begin{remark}
  If the factor spectral norms are uniformly bounded, then the diminishing-step
  guarantee also implies $\|G_A^k\|_*,\|G_B^k\|_*\to0$. See
  Appendix~\ref{app:convergence-proofs}.
\label{rem:factor_convergence}
\end{remark}

\paragraph{Directional equivalence with LoRA-Pro.}
LoRA-Pro is a popular and intuitive weight-space correction, so recovering it
places our surrogate within an established LoRA optimization framework.
\begin{proposition}[Directional equivalence with LoRA-Pro]
  \label{prop:lora-pro-equivalence}
  Under the Frobenius norm, let $X_F$ and $X_{\text{LoRA-Pro}}$ denote the
  weight-space directions induced by Eq.~\eqref{eq:lora-tangent-surrogate} and
  LoRA-Pro, respectively. Then
  \begin{equation}
    X_{\text{LoRA-Pro}}=\frac{\eta}{\lambda}X_F .
    \nonumber
  \end{equation}
  The updates coincide when $\eta=\lambda$.
\end{proposition}

This equivalence transfers the convergence analysis to LoRA-Pro, with
$\|\cdot\|_F$ replacing $\|\cdot\|_*$.
\begin{theorem}[Convergence of LoRA-Pro]
  \label{thm:lora-pro-convergence}
  Under the Frobenius-norm analogues of
  assumptions \textbf{(A1)}--\textbf{(A3)}, for the canonical $X=0$ LoRA-Pro
  representative defined in Appendix~\ref{app:convergence-proofs}, with step size
  $\eta \le \min\{1,\,1/(2C)\}$ and any
  $N \ge 1$,
  \begin{equation}
    \min_{0 \le k < N} \|\mathcal{P}_{T_k}(G_W^k)\|_F = O\bigl(N^{-1/2}\bigr) . \nonumber
    \label{eq:lora-pro-rate}
  \end{equation}
  Moreover, for the constant step size
  $\eta = \min\{1, 1/(2C)\}$, the iterates satisfy
  $\lim_{k\to\infty} \|\mathcal{P}_{T_k}(G_W^k)\|_F = 0$.
\end{theorem}
Proofs and constant definitions are given in
Appendix~\ref{app:convergence-proofs}.


\section{Conclusion}
We presented LoRA-TSD, which performs Muon-style spectral descent in the LoRA
tangent space through a factor-induced retraction. We established convergence
to projected stationarity and recovered LoRA-Pro as the Frobenius counterpart
with an analogous guarantee. Across six benchmarks and three model scales from
1B to 32B, LoRA-TSD outperforms generic and LoRA-specific optimizers and
remains stable across adapter ranks. The approximate inner oracle remains the main limitation
because it lacks an optimality guarantee. Controlled bounds and generative
evaluation remain future work.


\end{mainpart}

\begin{appendixpart}

\section{Proofs}
\label{app:convergence-proofs}

\subsection{Preliminaries}

We state two classical results used throughout the proofs.
In both theorems $\hat{M}=M+E$.

\begin{definition}[Principal angles]
  Let $\mathcal{U},\mathcal{V}\subseteq\mathbb{R}^n$ be subspaces of equal
  dimension $r$.  The \emph{principal angles}
  $0\le\theta_1\le\cdots\le\theta_r\le\pi/2$ between $\mathcal{U}$ and
  $\mathcal{V}$ are defined recursively by
  \[
  \begin{aligned}
    \cos\theta_k
    &=\max_{\substack{u\in\mathcal{U},\,v\in\mathcal{V}\\\|u\|=\|v\|=1}}
    u^\top v\\
    &\quad\text{s.t.}\;
    u\perp u_1,\ldots,u_{k-1},\;
    v\perp v_1,\ldots,v_{k-1}.
  \end{aligned}
  \]
  Equivalently, if $U,V\in\mathbb{R}^{n\times r}$ are orthonormal bases for
  $\mathcal{U}$ and $\mathcal{V}$, then $\cos\theta_k=\sigma_k(U^\top V)$.
  The largest principal angle $\theta_r$ satisfies
  $\sin\theta_r=\|(I-UU^\top)V\|_2$.
\end{definition}

\begin{theorem}[$\sin\Theta$ theorem~\citep{wedin1972perturbation}]
  \label{thm:wedin}
  Let $M,\hat{M}\in\mathbb{R}^{m\times n}$ with $\hat{M}=M+E$.
  Let $U_0,V_0$ (resp.\ $\hat{U}_0,\hat{V}_0$) be orthonormal bases for
  the leading $r$ left and right singular subspaces of $M$
  (resp.\ $\hat{M}$).
  Suppose $\sigma_r(M)-\sigma_{r+1}(\hat{M})\ge\Delta>0$.  Then
  \[
    \max\!\bigl\{
      \mathrm{dist}(\hat{U}_0,U_0),\,
      \mathrm{dist}(\hat{V}_0,V_0)
    \bigr\}
    \le
    \frac{\max\!\bigl\{\|EV_0\|_2,\,\|E^\top U_0\|_2\bigr\}}{\Delta}
    \le \frac{\|E\|_2}{\Delta},
  \]
  where $\mathrm{dist}(\hat{X}_0,X_0):=\|(I-X_0X_0^\top)\hat{X}_0\|_2$
  is the sine of the largest principal angle between the subspaces.
\end{theorem}

\begin{theorem}[Weyl's inequality]
  \label{thm:weyl}
  Let $M,\hat{M}\in\mathbb{R}^{m\times n}$ with $\hat{M}=M+E$.  Then
  \[
    |\sigma_i(M)-\sigma_i(\hat{M})|
    \le \|E\|_2,
    \qquad \forall\,i=1,\ldots,\min(m,n).
  \]
  Consequently,
  \[
    \sigma_r(M)-\sigma_{r+1}(\hat{M})
    \ge \sigma_r(M)-\sigma_{r+1}(M)-\|E\|_2.
  \]
\end{theorem}


\paragraph{Notation.}
At each outer iteration $k$ the algorithm produces:
\begin{itemize}
  \item $G_W^k:=\nabla\mathcal{L}(W_k)$ is the full-space gradient (never
    formed explicitly).
  \item $G_A^k:=B_k^\top G_W^k$ and $G_B^k:=G_W^k A_k^\top$ are the factor
    gradients, i.e.\ $\partial\mathcal{L}/\partial A_k$ and
    $\partial\mathcal{L}/\partial B_k$.
  \item $g_k:=\mathcal{P}_{T_k}(G_W^k)$ is the projection of the gradient
    onto the tangent space $T_k=\mathcal{T}_{W_k}\mathcal{M}_r$, which
    serves as the stationarity measure.
  \item $X_0:=\mathcal{P}_{T_k}(G_W^k)=g_k$ initializes the
    inner loop.
  \item $X_t:=\mathcal{P}_{T_k}(\mathrm{msign}(X_{t-1}))$,
    $t=1,\ldots,\tau$, are the alternating-projection iterates, and
    $X_\tau$ is the final inner-loop output used as the update direction.
  \item $\Delta W_k:=-\eta_k X_\tau$ is the full-space weight increment.
  \item $\Delta B_k,\Delta A_k$ are the factor increments recovered from
    $\Delta W_k$ via the canonical split
    (Eqs.~\eqref{eq:lora-delta-b}--\eqref{eq:lora-delta-a}).
\end{itemize}
We also define the constants
\[
  c_1:=\frac{1}{\delta^2},\quad
  c_2:=1+\sqrt{2}\,c_1,\quad
  C:=2G_{\max}rc_1+Mc_2^2.
\]


\subsection{Proof of Proposition~\ref{prop:transformation-invariance}}

We track the state through one pass of Algorithm~\ref{alg:lora-tsd-simple} and
show that each step produces the same result for $(B,A)$ and for the
reparametrized factors $(B',A')=(BR^{-1},RA)$, which by construction satisfy
$B'A'=BA=W-W_0$. We drop the iteration index $k$ and write $\eta$ for the step
size.

The two factorizations describe the same weight $W$,
so $G_W=\nabla_W\mathcal{L}(W)$ is identical for both. The factor gradients differ,
\[
  G_{A'}=(B')^\top G_W=R^{-\top}G_A,
  \qquad
  G_{B'}=G_W(A')^\top=G_B R^\top,
\]
but they enter the algorithm only through $G_W$, as we now use.

Consider the tangent projectors of
Eq.~\eqref{eq:lora-tangent-projector} for the primed factors. Using
$(A'(A')^\top)^{-1}=R^{-\top}(AA^\top)^{-1}R^{-1}$ and
$((B')^\top B')^{-1}=R(B^\top B)^{-1}R^\top$,
\[
\begin{aligned}
  P_{A'}
  &=(A')^\top(A'(A')^\top)^{-1}A'\\
  &=A^\top R^\top R^{-\top}(AA^\top)^{-1}R^{-1}RA
   =P_A,\\
  P_{B'}
  &=B'((B')^\top B')^{-1}(B')^\top\\
  &=BR^{-1}R(B^\top B)^{-1}R^\top R^{-\top}B^\top
   =P_B.
\end{aligned}
\]
Hence $\mathcal{P}_{T'}=\mathcal{P}_T$: the projector, and therefore the tangent
space itself, depends only on $W$.

The inner loop is unaffected as well. The initialization is
$X_0=\mathcal{P}_T(G_W)$, which is unchanged since both $\mathcal{P}_T$ and $G_W$
are. The iteration $X_{t+1}=\mathcal{P}_T(\mathrm{msign}(X_t))$ applies the same
operator $\mathcal{P}_T$ and the reparametrization-free map $\mathrm{msign}$, so by
induction every $X_t$, and in particular $X_\tau$, is identical for both
factorizations. Consequently the tangent update
$\Delta W=-\eta X_\tau$ is identical.

The reconstruction step transforms covariantly. With
$A^\dagger=A^\top(AA^\top)^{-1}$ and
$B^\dagger=(B^\top B)^{-1}B^\top$, the same manipulations as above give
$(A')^\dagger=A^\dagger R^{-1}$ and $(B')^\dagger=RB^\dagger$. Hence
\[
  \Delta B'=\Delta W(A')^\dagger=\Delta B\,R^{-1},
\]
and, since $\Delta B'A'=\Delta B R^{-1}RA=\Delta B A$,
\[
  \Delta A'=(B')^\dagger(\Delta W-\Delta B'A')
  =R\,B^\dagger(\Delta W-\Delta B A)=R\,\Delta A.
\]

Finally, the updated factors are
$B'+\Delta B'=(B+\Delta B)R^{-1}$ and $A'+\Delta A'=R(A+\Delta A)$, so
\[
  (B'+\Delta B')(A'+\Delta A')
  =(B+\Delta B)R^{-1}R(A+\Delta A)
  =(B+\Delta B)(A+\Delta A).
\]
Adding $W_0$ gives $W_{k+1}'=W_{k+1}$, which proves the claim. \qed


\subsection{Proof of Proposition~\ref{prop:lora-pro-equivalence}}

Consider the Frobenius-norm version of the tangent surrogate
Eq.~\eqref{eq:lora-surrogate-objective}:
\[
  \min_{\Delta A,\Delta B}
  \Bigl\{
    \langle G_A,\Delta A\rangle
    +
    \langle G_B,\Delta B\rangle
    +
    \frac{1}{2\lambda}
    \|\Delta B A+B\Delta A\|_F^2
  \Bigr\}.
\]
The linear terms rewrite as
\[
\begin{aligned}
  \langle G_A,\Delta A\rangle+\langle G_B,\Delta B\rangle
  =\langle B^\top G_W,\Delta A\rangle+\langle G_W A^\top,\Delta B\rangle
  =\langle G_W,\Delta B A+B\Delta A\rangle.
\end{aligned}
\]
Setting $\widetilde G:=\Delta B A+B\Delta A$ and completing the square,
\[
  \langle G_W,\widetilde G\rangle
  +\frac{1}{2\lambda}\|\widetilde G\|_F^2
  =
  \frac{1}{2\lambda}\|\widetilde G+\lambda G_W\|_F^2
  -\frac{\lambda}{2}\|G_W\|_F^2.
\]
Dropping the constant, the problem reduces to
\[
  \min_{\Delta A,\Delta B}\|\Delta B A+B\Delta A+\lambda G_W\|_F^2.
\]
Substituting $\widehat G_A:=-\Delta A/\lambda$ and
$\widehat G_B:=-\Delta B/\lambda$ gives
\[
  \|\Delta B A+B\Delta A+\lambda G_W\|_F^2
  =
  \lambda^2\|B\widehat G_A+\widehat G_B A-G_W\|_F^2.
\]
Since $\lambda^2>0$ does not affect the minimizers, the problem is equivalent to
\[
  \min_{\widehat G_A,\widehat G_B}
  \|B\widehat G_A+\widehat G_B A-G_W\|_F^2,
\]
which is precisely the auxiliary problem solved by LoRA-Pro. Its induced
weight-space direction is the orthogonal projection
\[
  B\widehat G_A^\star+\widehat G_B^\star A
  =\mathcal{P}_T(G_W).
\]
It follows from the substitution above that the surrogate solution is
\[
  X^\star=-\lambda\mathcal{P}_T(G_W).
\]
By contrast, a LoRA-Pro step with learning rate $\eta$ induces
$-\eta\mathcal{P}_T(G_W)=(\eta/\lambda)X^\star$. Thus the directions agree for
all positive $\eta$ and $\lambda$, while the increments are equal when
$\eta=\lambda$. \qed

\subsection{Setup for Theorem~\ref{thm:lora-tsd-convergence}}

\paragraph{Assumptions.}
We work under \textbf{(A1)}--\textbf{(A4)} and the theorem's step-size condition,
as stated in Section~\ref{subsec:convergence-guarantees}.
For convenience we recall the problem-specific ones.
\textbf{(A3)} requires the LoRA factors to stay full-rank:
$\sigma_r(A_k),\sigma_r(B_k)\ge\delta>0$.
\textbf{(A4)} is the descent-margin condition on the
alternating-projection output:
\[
  \langle g_k,X_\tau\rangle\ge c_0\|g_k\|_*,\qquad c_0>0.
\]
The step-size condition requires
$0<\eta_k\le\min\{1,\delta^2/(2\sqrt{2})\}$.

The following lemmas are used in the proof of
Theorem~\ref{thm:lora-tsd-convergence}.

\begin{lemma}[Row-space projector perturbation]
  \label{lem:projector-perturbation}
  Let $A \in \mathbb{R}^{r \times n}$ have rank $r$ with $\sigma_r(A)>0$,
  and let $A' = A + E$ satisfy $\|E\|_2 \le \sigma_r(A)/2$.
  Denote by $P_A = A^\top(AA^\top)^{-1}A$ and
  $P_{A'} = (A')^\top(A'(A')^\top)^{-1}A'$ the orthogonal projectors
  onto the row spaces of $A$ and $A'$, respectively.  Then
  \[
    \|P_{A'}-P_A\|_2 \le \frac{\|E\|_2}{\sigma_r(A)}.
  \]
\end{lemma}

\begin{proof}
Let $\theta_{\max}$ denote the largest principal angle between the row spaces
of $A$ and $A'$.
Write $P_{A'}-P_A=(I-P_A)P_{A'}-P_A(I-P_{A'})$.
For any unit vector $x$, $(I-P_A)P_{A'}x\in\mathrm{col}(I-P_A)$ and
$P_A(I-P_{A'})x\in\mathrm{col}(P_A)$ are orthogonal, so
\[
  \|(P_{A'}-P_A)x\|^2
  =
  \|(I-P_A)P_{A'}x\|^2+\|P_A(I-P_{A'})x\|^2.
\]
Let $Q_{A'}$ be any orthonormal basis for the row space of $A'$.
By \citet{bjorck1973numerical} (eq.~13), $\|(I-P_A)P_{A'}\|_2
=\|(I-P_A)Q_{A'}\|_2=\sin\theta_{\max}$,
and by symmetry $\|P_A(I-P_{A'})\|_2=\sin\theta_{\max}$.
Bounding each term by the corresponding spectral norm,
\[
\begin{aligned}
  \|(I-P_A)P_{A'}x\|^2
  &\le \|(I-P_A)P_{A'}\|_2^2\,\|P_{A'}x\|^2
  = \sin^2\theta_{\max}\,\|P_{A'}x\|^2,\\
  \|P_A(I-P_{A'})x\|^2
  &\le \|P_A(I-P_{A'})\|_2^2\,\|(I-P_{A'})x\|^2
  = \sin^2\theta_{\max}\,\|(I-P_{A'})x\|^2,
\end{aligned}
\]
and using $\|P_{A'}x\|^2+\|(I-P_{A'})x\|^2=\|x\|^2=1$,
\[
\begin{aligned}
  \|(P_{A'}-P_A)x\|^2
  \le
  \sin^2\theta_{\max}\bigl(\|P_{A'}x\|^2+\|(I-P_{A'})x\|^2\bigr)
  =\sin^2\theta_{\max},
\end{aligned}
\]
with equality at the corresponding principal vector, so
$\|P_{A'}-P_A\|_2=\sin\theta_{\max}$.

Apply Theorem~\ref{thm:wedin} with $M=A$, $\hat{M}=A'$,
taking $V_0$ to be the orthonormal basis for the row space of $A$.
Since $\mathrm{dist}(\hat{V}_0,V_0)=\sin\theta_{\max}$ and
$\Delta=\sigma_r(A)-\sigma_{r+1}(A')=\sigma_r(A)$,
\[
  \sin\theta_{\max}\le\frac{\|E\|_2}{\Delta}=\frac{\|E\|_2}{\sigma_r(A)}.
\]
\end{proof}

\begin{lemma}[Spectral norm of the alternating-projection output]
  \label{lem:xk-spectral-bound}
  For every outer iteration $k$, $\sigma_1(X_\tau)\le\sqrt{2}$.
\end{lemma}

\begin{proof}
For any unit vector $x$, the two components
$P_B Zx$ and $(I-P_B)ZP_Ax$
of $\mathcal{P}_T(Z)x$ are orthogonal. Hence
\[
\begin{aligned}
  \|\mathcal{P}_T(Z)x\|_2^2
  =\|P_BZx\|_2^2+\|(I-P_B)ZP_Ax\|_2^2
  \le\|Z\|_2^2(1+\|P_Ax\|_2^2)\le 2\|Z\|_2^2.
\end{aligned}
\]
Since $X_\tau=\mathcal{P}_{T_k}(\mathrm{msign}(X_{\tau-1}))$ and
$\|\mathrm{msign}(X_{\tau-1})\|_2=1$, we get
$\sigma_1(X_\tau)\le\sqrt{2}$.
\end{proof}

\begin{lemma}[Factor-remainder bound]
  \label{lem:remainder-bound}
  Under \textnormal{\textbf{(A3)}}, the factor increments satisfy
  \[
    \|\Delta B_k\|_2\le \frac{\sigma_1(\Delta W_k)}{\sigma_r(A_k)},
    \qquad
    \|\Delta A_k\|_2\le \frac{\sigma_1(\Delta W_k)}{\sigma_r(B_k)},
  \]
  and consequently
  \[
    \|\Delta B_k\Delta A_k\|_2
    \le c_1\,\sigma_1(\Delta W_k)^2
    \le 2c_1\eta_k^2,
  \]
  where $c_1=1/\delta^2$.
\end{lemma}

\begin{proof}
Recall the canonical factor-update split
(Eqs.~\eqref{eq:lora-delta-b}--\eqref{eq:lora-delta-a}):
\[
  \Delta B_k = \Delta W_k A_k^\dagger,\qquad
  \Delta A_k = B_k^\dagger(\Delta W_k - \Delta B_k A_k),
\]
where $A_k^\dagger=A_k^\top(A_kA_k^\top)^{-1}$ and
$B_k^\dagger=(B_k^\top B_k)^{-1}B_k^\top$ are the right and left
Moore–Penrose pseudoinverses, satisfying
$\|A_k^\dagger\|_2=1/\sigma_r(A_k)$ and $\|B_k^\dagger\|_2=1/\sigma_r(B_k)$.

By submultiplicativity,
\[
  \|\Delta B_k\|_2
  \le\|\Delta W_k\|_2\|A_k^\dagger\|_2
  =\frac{\sigma_1(\Delta W_k)}{\sigma_r(A_k)}.
\]

Since $\Delta B_kA_k=\Delta W_kP_{A_k}$, the second increment satisfies
$\Delta A_k=B_k^\dagger\Delta W_k(I-P_{A_k})$. Therefore,
\[
  \|\Delta A_k\|_2
  \le\|B_k^\dagger\|_2\|\Delta W_k\|_2\|I-P_{A_k}\|_2
  \le\frac{\sigma_1(\Delta W_k)}{\sigma_r(B_k)}.
\]

Multiplying the two bounds and using
$\sigma_r(A_k),\sigma_r(B_k)\ge\delta$,
\[
\begin{aligned}
  \|\Delta B_k\Delta A_k\|_2
  \le\|\Delta B_k\|_2\|\Delta A_k\|_2
  \le\frac{\sigma_1(\Delta W_k)^2}{\sigma_r(A_k)\sigma_r(B_k)}
  \le\frac{1}{\delta^2}\,\sigma_1(\Delta W_k)^2
  =c_1\,\sigma_1(\Delta W_k)^2.
\end{aligned}
\]
Finally, $\sigma_1(\Delta W_k)=\eta_k\sigma_1(X_\tau)\le\sqrt{2}\,\eta_k$
by Lemma~\ref{lem:xk-spectral-bound}, so
$c_1\sigma_1(\Delta W_k)^2\le 2c_1\eta_k^2$.
\end{proof}

\begin{lemma}[Lipschitz continuity of $\|g_k\|_*$]
  \label{lem:gk-lipschitz}
  Under \textnormal{\textbf{(A1)}}--\textnormal{\textbf{(A4)}} and the step-size
  condition,
  \[
    \bigl|\|g_{k+1}\|_*-\|g_k\|_*\bigr|\le L'\eta_k,
  \]
  where
  \[
    L'=4r\!\left(2Mc_2+\frac{2\sqrt{2}G_{\max}}{\delta^2}\right).
  \]
\end{lemma}

\begin{proof}
Decompose
\[
\begin{aligned}
  g_{k+1}-g_k
  =\mathcal{P}_{T_{k+1}}(G_W^{k+1})-\mathcal{P}_{T_k}(G_W^k)
  =\mathcal{P}_{T_{k+1}}(G_W^{k+1}-G_W^k)
    +(\mathcal{P}_{T_{k+1}}-\mathcal{P}_{T_k})(G_W^k).
\end{aligned}
\]

By Lemma~\ref{lem:xk-spectral-bound} and spectral Lipschitz continuity of
$\nabla\mathcal{L}$,
\[
\begin{aligned}
  \left\|\mathcal{P}_{T_{k+1}}(G_W^{k+1}-G_W^k)\right\|_2
  &\le\sqrt{2}\,\|G_W^{k+1}-G_W^k\|_2\\
  &\le\sqrt{2}\,M\,\|W_{k+1}-W_k\|_2\\
  &=\sqrt{2}\,M\,\|\Delta B_k A_k+B_k\Delta A_k+\Delta B_k\Delta A_k\|_2\\
  &\le\sqrt{2}\,M\bigl(\|\Delta W_k\|_2+\|\Delta B_k\Delta A_k\|_2\bigr)\\
  &\le\sqrt{2}\,M\bigl(\sqrt{2}\,\eta_k+2c_1\eta_k^2\bigr)\\
  &\le 2M\eta_k(1+\sqrt{2}\,c_1)=2Mc_2\eta_k,
\end{aligned}
\]
where we used $\|\Delta W_k\|_2=\eta_k\sigma_1(X_\tau)\le\sqrt{2}\,\eta_k$
(Lemma~\ref{lem:xk-spectral-bound}),
$\|\Delta B_k\Delta A_k\|_2\le c_1\sigma_1(\Delta W_k)^2\le 2c_1\eta_k^2$
(Lemma~\ref{lem:remainder-bound}),
and $\eta_k\le 1$ (the step-size condition).

Recall
$\mathcal{P}_T(Z)=P_BZ+ZP_A-P_BZP_A$.
Define
$\Delta P_A:=P_{A_{k+1}}-P_{A_k}$ and
$\Delta P_B:=P_{B_{k+1}}-P_{B_k}$.
Then
\[
\begin{aligned}
  (\mathcal{P}_{T_{k+1}}-\mathcal{P}_{T_k})(G_W^k)
  &=\Delta P_B G_W^k+G_W^k\Delta P_A
  -\bigl(P_{B_{k+1}}G_W^k P_{A_{k+1}}
  -P_{B_k}G_W^k P_{A_k}\bigr)\\
  &=\Delta P_B\,G_W^k(I-P_{A_{k+1}})
  +(I-P_{B_k})G_W^k\,\Delta P_A.
\end{aligned}
\]
Consequently,
\[
\begin{aligned}
  \left\|(\mathcal{P}_{T_{k+1}}-\mathcal{P}_{T_k})(G_W^k)\right\|_2
  \le \left\|\Delta P_B\,G_W^k(I-P_{A_{k+1}})\right\|_2
  + \left\|(I-P_{B_k})G_W^k\,\Delta P_A\right\|_2
  \le \bigl(\|\Delta P_A\|_2+\|\Delta P_B\|_2\bigr)G_{\max}.
\end{aligned}
\]

By Lemma~\ref{lem:projector-perturbation} (whose assumption
$\|\Delta A_k\|_2\le\sigma_r(A_k)/2$ is satisfied since
$\|\Delta A_k\|_2\le\sqrt{2}\eta_k/\delta\le\delta/2$
by Lemma~\ref{lem:remainder-bound} and the step-size condition),
\[
  \|\Delta P_A\|_2
  =
  \|P_{A_{k+1}}-P_{A_k}\|_2
  \le
  \frac{\|\Delta A_k\|_2}{\sigma_r(A_k)}.
\]
Using the factor-update bound
$\|\Delta A_k\|_2\le\sqrt{2}\eta_k/\delta$,
\[
  \|\Delta P_A\|_2
  \le
  \frac{\sqrt{2}}{\delta^2}\,\eta_k.
\]
Similarly,
\[
  \|\Delta P_B\|_2
  \le
  \frac{\|\Delta B_k\|_2}{\sigma_r(B_k)}
  \le
  \frac{\sqrt{2}}{\delta^2}\,\eta_k.
\]
Therefore,
\[
  \left\|(\mathcal{P}_{T_{k+1}}-\mathcal{P}_{T_k})(G_W^k)\right\|_2
  \le
  \frac{2\sqrt{2}G_{\max}}{\delta^2}\,\eta_k.
\]

Combining the two contributions,
\[
  \|g_{k+1}-g_k\|_2
  \le
  \left(2Mc_2+\frac{2\sqrt{2}G_{\max}}{\delta^2}\right)\eta_k.
\]
Since $\mathrm{rank}(g_{k+1}-g_k)\le 4r$,
\[
  \|g_{k+1}-g_k\|_*\le 4r\,\|g_{k+1}-g_k\|_2.
\]
Finally, by the reverse triangle inequality,
\[
\begin{aligned}
  \bigl|\|g_{k+1}\|_*-\|g_k\|_*\bigr|
  \le \|g_{k+1}-g_k\|_*
  \le 4r\left(2Mc_2+\frac{2\sqrt{2}G_{\max}}{\delta^2}\right)\eta_k
  = L'\eta_k.
\end{aligned}
\]
\end{proof}


\subsection{Proof of Theorem~\ref{thm:lora-tsd-convergence}}

\begin{theorem*}[Convergence of LoRA-TSD, restated]
  Let $\{W_k\}$ be the iterates produced by Algorithm~\ref{alg:lora-tsd-simple}
  (the momentum-free variant, Algorithm~\ref{alg:lora-tsd} differs only in
  the addition of momentum)
  and let $g_k=\mathcal{P}_{T_k}(G_W^k)$.
  Under assumptions \textbf{(A1)}--\textbf{(A4)} and the step-size condition,
  for any $N\ge 1$:
  \begin{enumerate}
    \item \emph{Constant step size.} If $\eta_k\equiv\eta$, then
    \[
      \min_{0\le k<N}\|g_k\|_*
      \le\frac{\mathcal{L}(W_0)-\mathcal{L}_*}{c_0\eta N}+\frac{C}{c_0}\eta,
    \]
    and for $\eta=\sqrt{(\mathcal{L}(W_0)-\mathcal{L}_*)/(CN)}$,
    \[
      \min_{0\le k<N}\|g_k\|_*
      \le\frac{2}{c_0}\sqrt{\frac{C(\mathcal{L}(W_0)-\mathcal{L}_*)}{N}}
      =O(N^{-1/2}).
    \]
    \item \emph{Diminishing step size.} If $\sum_k\eta_k=\infty$ and
    $\sum_k\eta_k^2<\infty$, then $\lim_{k\to\infty}\|g_k\|_*=0$.
  \end{enumerate}
\end{theorem*}

Since $X_\tau\in T_k$ and $\mathcal{P}_{T_k}=\mathcal{P}_{T_k}^*$,
\[
\begin{aligned}
  \langle G_W^k,\Delta W_k\rangle
  =-\eta_k\langle G_W^k,X_\tau\rangle
    =-\eta_k\langle\mathcal{P}_{T_k}(G_W^k),X_\tau\rangle
  =-\eta_k\langle g_k,X_\tau\rangle
  \le -\eta_k c_0\|g_k\|_*,
\end{aligned}
\]
where the last step uses \textbf{(A4)}.

By Lemma~\ref{lem:remainder-bound}, $\|\Delta B_k\Delta A_k\|_2\le 2c_1\eta_k^2$.
Combining with \textbf{(A1)} applied to the increment
$W_{k+1}-W_k=\Delta W_k+\Delta B_k\Delta A_k$,
where $\Delta W_k=\Delta B_k A_k+B_k\Delta A_k=-\eta_k X_\tau$ is the first-order factor increment,
\[
  \mathcal{L}(W_{k+1})-\mathcal{L}(W_k)
  \le \langle G_W^k,\Delta W_k+\Delta B_k\Delta A_k\rangle
  +\frac{M}{2}\|\Delta W_k+\Delta B_k\Delta A_k\|_2^2.
\]
The three terms are bounded as follows.
\begin{align*}
  \langle G_W^k,\Delta W_k\rangle
    &\le -\eta_k c_0\|g_k\|_*, \\
  |\langle G_W^k,\Delta B_k\Delta A_k\rangle|
    &\le G_{\max}r\sigma_1(\Delta B_k\Delta A_k)
      \le 2G_{\max}rc_1\eta_k^2, \\
  \|\Delta W_k+\Delta B_k\Delta A_k\|_2
    &\le \sqrt{2}\eta_k+2c_1\eta_k^2
      \le\sqrt{2}\,c_2\eta_k,
\end{align*}
giving
\[
  \mathcal{L}(W_{k+1})-\mathcal{L}(W_k)
  \le
  -\eta_k c_0\|g_k\|_*+C\eta_k^2.
  \label{eq:descent}
\]

\paragraph{Part (i): constant step size.}
Summing over $k=0,\ldots,N-1$ and telescoping,
\[
  c_0\eta\sum_{k=0}^{N-1}\|g_k\|_*
  \le
  \mathcal{L}(W_0)-\mathcal{L}(W_N)+C\eta^2 N.
\]
Since $\mathcal{L}(W_N)\ge\mathcal{L}_*$ by \textbf{(A2)},
\[
  c_0\eta\sum_{k=0}^{N-1}\|g_k\|_*
  \le
  \mathcal{L}(W_0)-\mathcal{L}_*+C\eta^2 N.
\]
Dividing by $c_0\eta N$ and using $\min\le\frac{1}{N}\sum$,
\[
  \min_{0\le k<N}\|g_k\|_*
  \le
  \frac{1}{N}\sum_{k=0}^{N-1}\|g_k\|_*
  \le
  \frac{\mathcal{L}(W_0)-\mathcal{L}_*}{c_0\eta N}+\frac{C\eta}{c_0}.
\]
The right-hand side is minimized over $\eta>0$ at
$\eta=\sqrt{(\mathcal{L}(W_0)-\mathcal{L}_*)/(CN)}$, giving
\[
  \min_{0\le k<N}\|g_k\|_*
  \le
  \frac{2}{c_0}\sqrt{\frac{C(\mathcal{L}(W_0)-\mathcal{L}_*)}{N}}
  =O(N^{-1/2}).
\]

\paragraph{Part (ii): diminishing step size.}
Summing the descent inequality over all $k$,
\[
  c_0\sum_{k=0}^{\infty}\eta_k\|g_k\|_*
  \le
  \mathcal{L}(W_0)-\mathcal{L}_*+C\sum_{k=0}^{\infty}\eta_k^2<\infty,
\]
so $\sum_k\eta_k\|g_k\|_*<\infty$.
If $\liminf_k\|g_k\|_*=2\varepsilon>0$, then $\|g_k\|_*>\varepsilon$
for all large $k$, and
$\sum_k\eta_k\|g_k\|_*\ge\varepsilon\sum_k\eta_k=\infty$, a contradiction.
Hence $\liminf_k\|g_k\|_*=0$.

To upgrade to $\lim_k\|g_k\|_*=0$, suppose for contradiction that
$\limsup_k\|g_k\|_*=s>0$ and fix $0<\varepsilon<s/3$.
Since $\liminf_k\|g_k\|_*=0$, we can extract an infinite sequence of
disjoint intervals $[m_j,\ell_j]$ defined as follows: let $m_j$ be the
last index before the $j$-th exceedance of $2\varepsilon$ at which
$\|g_{m_j}\|_*\le\varepsilon$, and let $\ell_j$ be the first subsequent
index with $\|g_{\ell_j}\|_*\ge 2\varepsilon$.
By construction, $\|g_k\|_*\ge\varepsilon$ for all $k\in(m_j,\ell_j]$.
Applying the triangle inequality and Lemma~\ref{lem:gk-lipschitz},
\[
\begin{aligned}
  \varepsilon
  \le\|g_{\ell_j}\|_*-\|g_{m_j}\|_*
  \le\sum_{k=m_j}^{\ell_j-1}\bigl|\|g_{k+1}\|_*-\|g_k\|_*\bigr|
  \le L'\sum_{k=m_j}^{\ell_j-1}\eta_k.
\end{aligned}
\]
Hence $\sum_{k=m_j}^{\ell_j-1}\eta_k\ge\varepsilon/L'$, and since
$\|g_k\|_*\ge\varepsilon$ on $(m_j,\ell_j]$,
\[
  \sum_{k=m_j}^{\ell_j-1}\eta_k\|g_k\|_*
  \ge\varepsilon\sum_{k=m_j+1}^{\ell_j-1}\eta_k
  \ge\frac{\varepsilon^2}{L'}-\varepsilon\,\eta_{m_j}.
\]
Since $\eta_k\to 0$, the right-hand side exceeds $\varepsilon^2/(2L')$ for
all large $j$.
Summing over the infinitely many disjoint intervals gives
$\sum_k\eta_k\|g_k\|_*=\infty$, contradicting the finiteness established
above.
Therefore $\lim_k\|g_k\|_*=0$. \qed

\paragraph{Factor-gradient consequence.}
The chain rule gives $G_A^k=B_k^\top G_W^k$. Left-multiplying
\[
  \mathcal{P}_{T_k}(G_W^k)
  =P_{B_k}G_W^k+G_W^kP_{A_k}-P_{B_k}G_W^kP_{A_k}
\]
by $B_k^\top$ and using $B_k^\top P_{B_k}=B_k^\top$ yields
\[
  B_k^\top\mathcal{P}_{T_k}(G_W^k)
  =B_k^\top G_W^k+G_A^kP_{A_k}-G_A^kP_{A_k}
  =G_A^k.
\]
Analogously, $G_B^k=\mathcal{P}_{T_k}(G_W^k)A_k^\top$. If
$\sigma_1(A_k),\sigma_1(B_k)\le\Gamma$, then
\[
  \|G_A^k\|_*,\|G_B^k\|_*
  \le \Gamma\|\mathcal{P}_{T_k}(G_W^k)\|_*,
\]
which proves Remark~\ref{rem:factor_convergence}.


\subsection{Setup for Theorem~\ref{thm:lora-pro-convergence}}

\begin{algorithm}{LoRA-Pro}
\label{alg:lora-pro}
\begin{algorithmic}[1]
  \For{$k=0,1,\ldots$}
    \State $G_A^k\gets\nabla_A\mathcal{L}(W_k)$,\;
      $G_B^k\gets\nabla_B\mathcal{L}(W_k)$
    \State $G_{A_k}^*\gets B_k^\dagger G_W^k$, \;
      $G_{B_k}^*\gets(I-P_{B_k})G_W^kA_k^\dagger$
      \Comment{canonical $X=0$ minimizer}
    \State $A_{k+1}\gets A_k-\eta_k G_{A_k}^*$,\;
      $B_{k+1}\gets B_k-\eta_k G_{B_k}^*$
  \EndFor
\end{algorithmic}
\end{algorithm}

\paragraph{Assumptions.}
The proof uses \textbf{(A1)}--\textbf{(A3)} and the step-size condition. The
descent-margin condition \textbf{(A4)} is not needed because LoRA-Pro solves
the inner problem exactly.
For convenience we recall the relevant ones.
\textbf{(A1)} requires Frobenius smoothness and gradient boundedness:
$\|\nabla\mathcal{L}(W)\|_F\le G_{\max}$.
\textbf{(A2)} requires $\mathcal{L}$ to be bounded below.
\textbf{(A3)} requires the factors to stay full-rank:
$\sigma_r(A_k),\sigma_r(B_k)\ge\delta>0$.
Here the step-size condition specialises to
$\eta_k\le\min\bigl\{1,1/(2C)\bigr\}$ with $C$ as defined below.

\paragraph{Explicit solution of the inner problem.}
The following result from \citet{lora-pro} gives the closed-form minimizer
of the LoRA-Pro inner problem.

\begin{lemma}[\citealt{lora-pro}, Theorem~2.1]
  \label{lem:lora-pro-solution}
  Let $B\in\mathbb{R}^{m\times r}$ and $A\in\mathbb{R}^{r\times n}$ be
  full-rank. The optimal solution of
  $\min_{G_A,G_B}\|BG_A+G_BA-G_W\|_F^2$
  is
  \[
    G_A^*=B^\dagger G_W+XA,
    \qquad
    G_B^*=(I-P_B)G_WA^\dagger-BX,
  \]
  where $X\in\mathbb{R}^{r\times r}$ is an arbitrary matrix,
  $B^\dagger=(B^\top B)^{-1}B^\top$, $A^\dagger=A^\top(AA^\top)^{-1}$,
  and $P_B=B(B^\top B)^{-1}B^\top$.
\end{lemma}

The free matrix $X$ reflects a gauge ambiguity of the factor updates. Indeed,
adding $XA$ to $G_A^*$ and subtracting $BX$ from $G_B^*$ leaves the induced
tangent direction unchanged because the two terms $BXA$ cancel. We fix this
ambiguity by choosing $X=0$, which gives
$G_{A_k}^*=B_k^\dagger G_W^k$ and
$G_{B_k}^*=(I-P_{B_k})G_W^k A_k^\dagger$. This choice yields the canonical
orthogonal split
\[
  B_kG_{A_k}^*=P_{B_k}G_W^k,
  \qquad
  G_{B_k}^*A_k=(I-P_{B_k})G_W^kP_{A_k},
\]
whose two components have orthogonal column spaces. Thus $X=0$ removes only a
cancelling factor-space component and preserves the LoRA-Pro update in weight
space. Since the second-order product of the factor increments depends on the
chosen representative, the convergence result below concerns this explicit
canonical choice.

We also define
\[
  C_1:=\frac{1}{\delta^2},\quad
  C_3:=1+C_1G_{\max},\quad
  C:=C_1G_{\max}+\tfrac{M}{2}C_3^2.
\]


\subsection{Proof of Theorem~\ref{thm:lora-pro-convergence}}

Let $G_W^k:=\nabla\mathcal{L}(W_k)$ and $g_k:=\mathcal{P}_{T_k}(G_W^k)$.

By the LoRA-Pro update rule,
$A_{k+1}=A_k-\eta G_{A_k}^*$ and $B_{k+1}=B_k-\eta G_{B_k}^*$,
where $(G_{A_k}^*,G_{B_k}^*)$ is the minimizer of
$\|G_W^k-(B_kG_A+G_BA_k)\|_F^2$.
Since the feasible set is exactly the tangent space $T_k$, the minimizer
satisfies
\[
  B_kG_{A_k}^*+G_{B_k}^*A_k=\mathcal{P}_{T_k}(G_W^k)=g_k.
\]
Therefore the first-order inner product evaluates to
\[
\begin{aligned}
  \langle G_W^k,-\eta(B_kG_{A_k}^*+G_{B_k}^*A_k)\rangle
  &=-\eta\langle G_W^k,g_k\rangle
  =-\eta\langle\mathcal{P}_{T_k}(G_W^k),g_k\rangle
  =-\eta\|g_k\|_F^2.
\end{aligned}
\]

By Lemma~\ref{lem:lora-pro-solution} with $X=0$,
$G_{A_k}^*=B_k^\dagger G_W^k$ and
$G_{B_k}^*=(I-P_{B_k})G_W^k A_k^\dagger$.
We now show these can be expressed in terms of $g_k$.

For $G_{A_k}^*$: using
$\mathcal{P}_{T_k}(G_W^k)=P_{B_k}G_W^k+(I-P_{B_k})G_W^kP_{A_k}$
and $B_k^\dagger(I-P_{B_k})=0$,
\[
  B_k^\dagger g_k
  =B_k^\dagger P_{B_k}G_W^k
  +B_k^\dagger(I-P_{B_k})G_W^kP_{A_k}
  =B_k^\dagger G_W^k=G_{A_k}^*.
\]
Hence $G_{A_k}^*=B_k^\dagger g_k$, and
\[
  \|G_{A_k}^*\|_F
  =\|B_k^\dagger g_k\|_F
  \le\|B_k^\dagger\|_2\|g_k\|_F
  =\frac{\|g_k\|_F}{\sigma_r(B_k)}
  \le\frac{\|g_k\|_F}{\delta}.
\]

Next, we estimate $G_{B_k}^*$.
\[
\begin{aligned}
  G_W^k - g_k
  &= G_W^k - P_{B_k}G_W^k - (I-P_{B_k})G_W^kP_{A_k}\\
  &=(I-P_{B_k})G_W^k-(I-P_{B_k})G_W^kP_{A_k}
    =(I-P_{B_k})G_W^k(I-P_{A_k}).
\end{aligned}
\]
Multiplying from the left by $I-P_{B_k}$ and from the right by $A_k^\dagger$,
\[
  (I-P_{B_k})(G_W^k-g_k)A_k^\dagger
  = (I-P_{B_k})^2 G_W^k(I-P_{A_k})A_k^\dagger.
\]
Since $(I-P_{A_k})A_k^\dagger=0$, it follows that
$(I-P_{B_k})(G_W^k-g_k)A_k^\dagger=0$, i.e.\
\[
  (I-P_{B_k})G_W^kA_k^\dagger = (I-P_{B_k})g_kA_k^\dagger.
\]
Hence $G_{B_k}^*=(I-P_{B_k})G_W^kA_k^\dagger=(I-P_{B_k})g_kA_k^\dagger$ and
\[
\begin{aligned}
  \|G_{B_k}^*\|_F
  &=\|(I-P_{B_k})G_W^kA_k^\dagger\|_F
  =\|(I-P_{B_k})g_kA_k^\dagger\|_F\\
  &\le\|I-P_{B_k}\|_2\|g_k\|_F\|A_k^\dagger\|_2
  =\frac{\|g_k\|_F}{\sigma_r(A_k)}\le\frac{\|g_k\|_F}{\delta}.
\end{aligned}
\]

It follows that
\[
  \|\eta^2 G_{B_k}^*G_{A_k}^*\|_F
  \le\eta^2\|G_{B_k}^*\|_F\|G_{A_k}^*\|_F
  \le\frac{\eta^2}{\delta^2}\|g_k\|_F^2=C_1\eta^2\|g_k\|_F^2.
\]

\[
\begin{aligned}
  W_{k+1}-W_k
  &=(B_k-\eta G_{B_k}^*)(A_k-\eta G_{A_k}^*)-B_kA_k\\
  &=-\eta(G_{B_k}^*A_k+B_kG_{A_k}^*)
    +\eta^2G_{B_k}^*G_{A_k}^*
    =-\eta g_k+\eta^2G_{B_k}^*G_{A_k}^*.
\end{aligned}
\]
Since $\|g_k\|_F\le\|G_W^k\|_F\le G_{\max}$ and $\eta\le 1$,
\[
  \|W_{k+1}-W_k\|_F
  \le\eta\|g_k\|_F+C_1\eta^2\|g_k\|_F^2
  \le\eta\|g_k\|_F(1+C_1G_{\max})=C_3\eta\|g_k\|_F.
\]

By Frobenius smoothness \textbf{(A1)},
\[
\begin{aligned}
  \mathcal{L}(W_{k+1})-\mathcal{L}(W_k)
  &\le\langle G_W^k,W_{k+1}-W_k\rangle
    +\frac{M}{2}\|W_{k+1}-W_k\|_F^2\\
  &=-\eta\|g_k\|_F^2
    +\eta^2\langle G_W^k,G_{B_k}^*G_{A_k}^*\rangle
  +\frac{M}{2}C_3^2\eta^2\|g_k\|_F^2.
\end{aligned}
\]
Since
\[
  |\langle G_W^k,G_{B_k}^*G_{A_k}^*\rangle|
  \le\|G_W^k\|_F\|G_{B_k}^*G_{A_k}^*\|_F
  \le C_1G_{\max}\|g_k\|_F^2,
\]
\[
  \mathcal{L}(W_{k+1})-\mathcal{L}(W_k)
  \le-\eta\|g_k\|_F^2+C\eta^2\|g_k\|_F^2
  =-\eta(1-C\eta)\|g_k\|_F^2.
\]
For $\eta\le\min\{1,1/(2C)\}$, $1-C\eta\ge 1/2$, so
\[
  \mathcal{L}(W_{k+1})-\mathcal{L}(W_k)\le-\frac{\eta}{2}\|g_k\|_F^2.
\]

Summing over $k=0,\ldots,N-1$ and using $\mathcal{L}(W_N)\ge\mathcal{L}_*$,
\[
  \frac{\eta}{2}\sum_{k=0}^{N-1}\|g_k\|_F^2
  \le\mathcal{L}(W_0)-\mathcal{L}_*.
\]
Since $\min\le\frac{1}{N}\sum$ and $\min\|g_k\|_F\le(\min\|g_k\|_F^2)^{1/2}$,
\[
  \min_{0\le k<N}\|g_k\|_F
  \le\sqrt{\frac{2(\mathcal{L}(W_0)-\mathcal{L}_*)}{\eta N}}
  =O(N^{-1/2}).
\]
Letting $N\to\infty$ gives $\sum_{k=0}^\infty\|g_k\|_F^2<\infty$,
hence $\lim_{k\to\infty}\|g_k\|_F=0$. \qed


\section{Retraction Analysis}
\label{app:retraction}

This appendix proves that the factor-induced map used by LoRA-TSD is a
retraction on the fixed-rank manifold, and compares it with the truncated-SVD
retraction of Riemannion.

We work on the fixed-rank manifold
$\mathcal{M}_r=\{W\in\mathbb{R}^{m\times n}:\mathrm{rank}\,W=r\}$. The constant
affine shift $W_0$ does not affect the tangent geometry, so we omit it. Fix a
point $W=BA\in\mathcal{M}_r$ with full-rank factors $B\in\mathbb{R}^{m\times r}$
and $A\in\mathbb{R}^{r\times n}$, and recall the tangent space
\[
  T_W\mathcal{M}_r
  =\{\Delta B\,A+B\,\Delta A:
    \Delta B\in\mathbb{R}^{m\times r},\ \Delta A\in\mathbb{R}^{r\times n}\}.
\]
We use the standard notion of a retraction~\citep{absil2008optimization}. A map
$R_W:T_W\mathcal{M}_r\to\mathcal{M}_r$ that is smooth on a neighborhood of the
origin is a retraction at $W$ when
\begin{enumerate}
  \item[(i)] $R_W(0)=W$ \quad(centering), and
  \item[(ii)] $\tfrac{d}{dt}R_W(t\xi)\big|_{t=0}=\xi$ for every
    $\xi\in T_W\mathcal{M}_r$ \quad(local rigidity).
\end{enumerate}

\paragraph{The factor-induced map.}
Given a tangent direction $\xi\in T_W\mathcal{M}_r$, LoRA-TSD reconstructs the
factor increments by the canonical split of
Eqs.~\eqref{eq:lora-delta-b}--\eqref{eq:lora-delta-a},
\[
  \Delta B=\xi A^\dagger,
  \qquad
  \Delta A=B^\dagger\bigl(\xi-\Delta B\,A\bigr),
\]
with $A^\dagger=A^\top(AA^\top)^{-1}$ and $B^\dagger=(B^\top B)^{-1}B^\top$. Both
increments depend linearly on $\xi$. LoRA-TSD then applies the update
\[
  R_W(\xi):=\bigl(B+\Delta B\bigr)\bigl(A+\Delta A\bigr).
\]

\begin{proposition}[The factor-induced map is a retraction]
  \label{prop:factor-retraction}
  If $B$ and $A$ have full rank, then $R_W$ is a retraction at $W=BA$ in the
  sense of \citet{absil2008optimization}.
\end{proposition}

\begin{proof}
For smoothness, note that the increments $\Delta B$ and $\Delta A$ depend
linearly, hence smoothly, on $\xi$. The full-rank matrices form an open set, so
for small $\|\xi\|_2$ both $B+\Delta B$ and $A+\Delta A$ remain full rank. Their
product then has rank exactly $r$ and lies in $\mathcal{M}_r$. Hence $R_W$ is
smooth from a neighborhood of $0\in T_W\mathcal{M}_r$ into $\mathcal{M}_r$.

Centering holds because at $\xi=0$ both increments vanish, so $R_W(0)=BA=W$.

For local rigidity, expand the product,
\[
  R_W(\xi)
  =BA+\bigl(\Delta B\,A+B\,\Delta A\bigr)+\Delta B\,\Delta A .
\]
The reconstructed increments satisfy
$\Delta B\,A+B\,\Delta A=\xi P_A+P_B\xi(I-P_A)=\mathcal{P}_T(\xi)=\xi$, using
$A^\dagger A=P_A$, $BB^\dagger=P_B$, the projector
formula~\eqref{eq:lora-tangent-projector}, and $\mathcal{P}_T(\xi)=\xi$ for
$\xi\in T_W\mathcal{M}_r$. Hence $R_W(\xi)=W+\xi+\Delta B\,\Delta A$.
Because the increments are linear in $\xi$, the scaled direction $t\xi$
reconstructs to $t\,\Delta B$ and $t\,\Delta A$, so
\[
  R_W(t\xi)=\bigl(B+t\,\Delta B\bigr)\bigl(A+t\,\Delta A\bigr)
  =W+t\,\xi+t^2\,\Delta B\,\Delta A.
\]
The second-order coefficient is finite. Since
$\Delta B\,A=\xi P_A$, we have
$\Delta A=B^\dagger\xi(I-P_A)$. Therefore,
\[
  \|\Delta B\|_2\le\frac{\|\xi\|_2}{\sigma_r(A)},
  \qquad
  \|\Delta A\|_2\le\frac{\|\xi\|_2}{\sigma_r(B)},
\]
so that
\[
  \|\Delta B\,\Delta A\|_2
  \le\frac{\|\xi\|_2^2}{\sigma_r(A)\sigma_r(B)},
\]
the same estimate as Lemma~\ref{lem:remainder-bound}. Hence
$R_W(t\xi)=W+t\,\xi+O(t^2)$, and
$\tfrac{d}{dt}R_W(t\xi)\big|_{t=0}=\xi$ for every $\xi\in T_W\mathcal{M}_r$, which
is condition (ii).
\end{proof}

\begin{table*}[t]
\centering
\small
\setlength{\tabcolsep}{4pt}
\begin{tabular}{rr cc c l}
\toprule
$m$ & $n$ & Factor-induced (LoRA-TSD) & Truncated SVD~\citep{riemannion} & Speedup & Layer \\
\midrule
4096  & 4096 & $0.434$ & $1.025$ & $2.36\times$ & \texttt{q\_proj} \\
1024  & 4096 & $0.434$ & $0.911$ & $2.10\times$ & \texttt{k/v\_proj} \\
14336 & 4096 & $0.431$ & $1.215$ & $2.82\times$ & \texttt{gate/up\_proj} \\
\bottomrule
\end{tabular}
\caption{Wall-clock time (ms) of a single retraction step,
  factor-induced (LoRA-TSD) versus truncated-SVD (Riemannion), at the weight shapes of
  Llama-3.1-8B, measured on an NVIDIA H200.}
\label{tab:retraction-time}
\end{table*}

\paragraph{Relation to the truncated-SVD retraction.}
The connection to Riemannion has two parts: the tangent direction and the
retraction. For $\tau=1$, LoRA-TSD recovers the tangent-direction part of
Riemannion, since both methods apply a matrix-sign orthogonalization to the
tangent-projected gradient. For $\tau>1$, LoRA-TSD repeats the tangent projection
and matrix-sign steps as an alternating-projection refinement toward the
constrained spectral oracle of Eq.~\eqref{eq:lora-constrained-tangent}. In our
hyperparameter tuning, $\tau=3$ to $5$ performs markedly better than $\tau=1$,
so this refinement accounts for part of the empirical gain.

Riemannion instead uses the metric-projection retraction
$R_W^{\mathrm{SVD}}(\xi)=\mathrm{SVD}_r(W+\xi)$, the best rank-$r$ approximation
of $W+\xi$. That this map is a smooth retraction on $\mathcal{M}_r$ is
standard~\citep{vandereycken2013lowrank}. Being retractions, both $R_W$ and
$R_W^{\mathrm{SVD}}$ satisfy $R_W(\xi)=W+\xi+O(\|\xi\|_2^2)$, so they agree to
first order in $\xi$ and can differ only at second order. For the factor-induced
map the second-order term is exactly $\Delta B\,\Delta A$, whereas for the
truncated-SVD map it is the curvature correction of $\mathcal{M}_r$. The two
therefore produce different iterates whenever $\xi\ne 0$, even though their
tangent directions coincide. In particular, equal tangent directions at
$\tau=1$ do not make LoRA-TSD and Riemannion identical after retraction.
The factor-induced retraction costs only the single product
$\Delta B\,\Delta A$ and follows directly from the LoRA parameterization, while
the truncated-SVD retraction requires an SVD at every step.
Table~\ref{tab:retraction-time} reports the wall-clock cost of a single
retraction step at the attention and MLP weight shapes of Llama-3.1-8B. The
factor-induced retraction is essentially independent of the layer size, whereas
the truncated-SVD cost grows with it, so the gap widens for the larger
projections.


\section{Experimental Details and Additional Results}
\label{app:experiments}

\paragraph{Datasets.}
We evaluate Llama-3.2-1B-Instruct and Llama-3.1-8B on six English benchmarks
covering commonsense reasoning and natural language inference. We also evaluate
Qwen3-32B on SIQA and OBQA. We report accuracy for every benchmark because each
test example has a discrete target, making correctness unambiguous. The label
distributions are sufficiently balanced for accuracy to provide an appropriate
summary of performance.
\begin{itemize}
  \item \textbf{BoolQ}~\citep{clark2019boolq} is a reading-comprehension dataset
    of naturally occurring yes/no questions, each paired with a Wikipedia passage
    from which the model predicts a yes or no answer.
  \item \textbf{PIQA}~\citep{bisk2020piqa} tests physical commonsense: given a
    goal, the model chooses which of two candidate solutions is physically
    plausible.
  \item \textbf{SIQA}~\citep{sap2019socialiqa} (Social IQa) probes social and
    emotional commonsense, asking the model to reason about the motivations,
    reactions, and consequences of everyday social interactions.
  \item \textbf{OBQA}~\citep{mihaylov2018openbookqa} (OpenBookQA) is a
    multiple-choice science question-answering dataset modeled on open-book
    exams, requiring elementary science facts to be combined with broad
    commonsense.
  \item \textbf{QNLI}~\citep{wang2019glue} is the question-answering
    natural-language-inference task of the GLUE benchmark, derived from
    SQuAD~\citep{rajpurkar2016squad}. Each example pairs a question with a
    sentence, and the model decides whether the sentence contains the answer.
  \item \textbf{MultiNLI}~\citep{williams2018multinli} is a multi-genre
    natural-language-inference corpus in which the model classifies the relation
    between a premise and a hypothesis as entailment, contradiction, or neutral.
\end{itemize}

\paragraph{Hyperparameter tuning protocol.}
For each optimizer and model size, we first evaluated three or four learning
rates on BoolQ. We then ran three additional trials around the best value from
this coarse sweep. We selected the learning rate by the highest validation
score reached at any point during training. The selected configuration was
frozen and reused without further tuning on all remaining benchmarks. This
protocol avoids per-dataset tuning and keeps the comparison focused on the
optimizer. The selected values are reported in Tables~\ref{tab:hparams-1b}
and~\ref{tab:hparams-8b}.
For Qwen3-32B, we transfer the Llama-3.1-8B settings with minimal tuning and
reduce gradient clipping from $1$ to $0.25$ for training stability.

\begin{wrapfigure}[17]{r}{0.58\textwidth}
  \centering
  \input{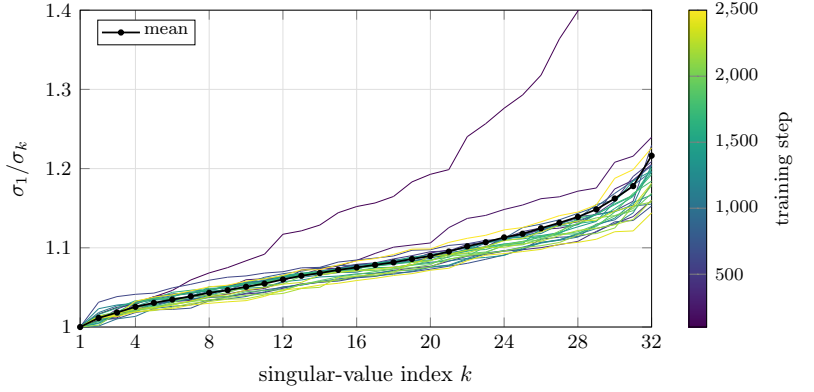}
  \captionsetup{font=normalsize}
  \caption{Per-step singular-value ratios $\sigma_1/\sigma_k$ of the LoRA-TSD
    update, colored by training step, with the mean highlighted. The spectrum
    stays close to flat across all steps, reflecting the spectral-norm geometry
    of the tangent-space update.}
  \label{fig:singular_values}
\end{wrapfigure}

\paragraph{Rank ablation.}
Table~\ref{tab:rank-ablation} gives the full rank sweep on
BoolQ and SIQA with Llama-3.1-8B, extending the summary in
Section~\ref{sec:experiments}. All hyperparameters other than the rank $r$ are held fixed. LoRA-TSD
achieves the best accuracy in 11 of 14 settings and stays within a narrow
accuracy band as $r$ varies. LoRA-Rite is competitive at small ranks but
degrades at larger ranks, including a sharp drop at $r=64$ on SIQA.

\paragraph{Spectrum of the update.}
Does the tangent-space update produce the balanced spectrum that motivates the
spectral-norm geometry? Figure~\ref{fig:singular_values} plots the ratio
$\sigma_1/\sigma_k$ of the LoRA-TSD update against the singular-value index $k$
for every recorded training step. The LoRA rank here is $r=16$, so
$\sigma_{32}$ is the last nonzero singular value of the tangent update
$\Delta B_k A_k + B_k \Delta A_k$, which has rank at most $2r=32$. The ratio
stays close to one along the whole spectrum, reaching only about $1.2$ on average
at the smallest singular value, and the per-step curves remain tightly clustered
throughout training. The update is therefore almost equalized rather than
dominated by a few leading directions, which is what the spectral-norm geometry
inherited from Muon is meant to produce.

\begin{table*}[h]
\centering
\small
\setlength{\tabcolsep}{6pt}
\begin{tabular}{r rrrr rrrr}
\toprule
& \multicolumn{4}{c}{BoolQ} & \multicolumn{4}{c}{SIQA} \\
\cmidrule(lr){2-5}\cmidrule(lr){6-9}
$r$ & \textbf{LoRA-TSD} & Muon & LoRA-Rite & AdamW
    & \textbf{LoRA-TSD} & Muon & LoRA-Rite & AdamW \\
\midrule
1  & \textbf{89.82} & 65.63 & 88.81 & 88.65 & \textbf{82.14} & 80.55 & 81.78 & 80.96 \\
2  & 89.17 & 64.07 & \textbf{89.66} & 88.96 & 82.60 & 80.71 & \textbf{82.65} & 81.78 \\
4  & \textbf{89.69} & 89.57 & 89.66 & 88.87 & 82.50 & 80.86 & \textbf{82.55} & 80.86 \\
8  & \textbf{90.49} & 89.82 & 63.70 & 89.48 & \textbf{83.47} & 81.53 & 81.73 & 81.73 \\
16 & \textbf{90.46} & 89.14 & 88.93 & 89.60 & \textbf{83.06} & 81.20 & 81.40 & 81.80 \\
32 & \textbf{89.79} & 89.14 & 87.86 & 89.02 & \textbf{82.45} & 81.01 & 81.06 & 81.83 \\
64 & \textbf{90.06} & 89.11 & 71.93 & 89.97 & \textbf{82.29} & 81.06 & 77.84 & 81.83 \\
\bottomrule
\end{tabular}
\caption{Rank ablation on BoolQ and SIQA (Llama-3.1-8B). Accuracy versus LoRA
  rank $r$ with all other hyperparameters fixed. Bold marks the best method per
  rank within each dataset.}
\label{tab:rank-ablation}
\end{table*}

\begin{table*}[t]
\centering
\small
\setlength{\tabcolsep}{6pt}
\begin{tabular}{l c l}
\toprule
Method & $\eta$ & Method-specific hyperparameters \\
\midrule
LoRA-SGD      & $10^{-3}$        & momentum $0.9$ \\
AdamW         & $10^{-3}$        & $\beta=(0.9,0.999)$, $\epsilon=10^{-8}$ \\
Muon          & $10^{-2}$        & momentum $0.95$, Nesterov, NS steps $5$ \\
Riem-SGD      & $2{\cdot}10^{-3}$ & momentum $0.95$ \\
LoRA-Rite     & $8{\cdot}10^{-5}$ & $\beta=(0.9,0.999)$, $\epsilon=10^{-6}$ \\
\textbf{LoRA-TSD} & $5{\cdot}10^{-3}$ & $\tau=5$, NS steps $5$ \\
Riemannion    & $3{\cdot}10^{-5}$ & momentum $0.95$, NS steps $6$ \\
LoRA-Pro      & $5{\cdot}10^{-4}$ & $\epsilon_{\mathrm{inv}}=10^{-6}$ \\
\midrule
Muon Full-FT  & $3{\cdot}10^{-2}$ & momentum $0.95$, Nesterov, NS steps $5$ \\
\bottomrule
\end{tabular}
\caption{Selected hyperparameters for Llama-3.2-1B-Instruct, tuned on BoolQ.
  Common to all methods: rank $r=16$, $\alpha=16$, dropout $0.05$, effective batch
  size $32$, $3000$ steps, linear schedule, adapters on
  \texttt{q\_proj}/\texttt{k\_proj}/\texttt{v\_proj}.}
\label{tab:hparams-1b}
\end{table*}

\begin{table*}[t]
\centering
\small
\setlength{\tabcolsep}{6pt}
\begin{tabular}{l c l}
\toprule
Method & $\eta$ & Method-specific hyperparameters \\
\midrule
LoRA-SGD      & $5{\cdot}10^{-3}$ & momentum $0.9$ \\
AdamW         & $10^{-4}$        & $\beta=(0.9,0.999)$, $\epsilon=10^{-8}$ \\
Muon          & $10^{-4}$        & momentum $0.95$, Nesterov, NS steps $5$ \\
Riem-SGD      & $2{\cdot}10^{-3}$ & momentum $0.95$ \\
LoRA-Rite     & $10^{-4}$        & $\beta=(0.9,0.999)$, $\epsilon=10^{-6}$ \\
\textbf{LoRA-TSD} & $4{\cdot}10^{-3}$ & $\tau=5$, NS steps $5$, rebalance every $250$ \\
Riemannion    & $3{\cdot}10^{-5}$ & momentum $0.95$, NS steps $6$ \\
LoRA-Pro      & $5{\cdot}10^{-5}$ & $\epsilon_{\mathrm{inv}}=10^{-6}$ \\
\bottomrule
\end{tabular}
\caption{Selected hyperparameters for Llama-3.1-8B, tuned on BoolQ.
  Common to all methods: rank $r=16$, $\alpha=16$, dropout $0.05$, effective batch
  size $32$, $3000$ steps, linear schedule, adapters on
  \texttt{q\_proj}/\texttt{k\_proj}/\texttt{v\_proj}.}
\label{tab:hparams-8b}
\end{table*}


\end{appendixpart}
\end{document}